\documentclass[sigconf]{acmart}
\usepackage{graphicx}
\usepackage{amsmath,amsfonts}
\usepackage{mathtools}
\usepackage{bm}
\usepackage{subcaption}
\usepackage{booktabs}
\usepackage{multirow}
\usepackage{algorithm}
\usepackage{algpseudocode} 
\usepackage{tikz}
\usepackage{pgfplots}
\usepackage{xcolor}
\usepackage{placeins}
\usepackage{float}
\usepackage{balance}
\usepackage{enumerate}
\usetikzlibrary{intersections}
\usepackage{tabularx}
\usepackage{booktabs}
\usepackage{pgfplots}
\pgfplotsset{compat=newest}
\usepackage{caption}

\theoremstyle{plain}
\newtheorem{theorem}{Theorem}[section]
\newtheorem{proposition}[theorem]{Proposition}
\newtheorem{corollary}[theorem]{Corollary}
\theoremstyle{definition}
\newtheorem{definition}[theorem]{Definition}
\newtheorem{example}{Example}
\theoremstyle{remark}
\newtheorem{remark}[theorem]{Remark}
\newtheorem{lemma}[theorem]{Lemma}

\usepackage{appendix}
\usepackage{booktabs}
\usepackage{caption}
\usepackage{listings}
\usepackage{microtype}
\usepackage{pgfplots}
\usepackage{tabularx}
\usepackage{tikz}
\usepackage{xcolor}
\usetikzlibrary{calc}
\usetikzlibrary{external}
\usetikzlibrary{plotmarks}
\usepgfplotslibrary{fillbetween}
\pgfplotsset{compat=1.15}
\lstdefinelanguage{json}{%
basicstyle={\normalfont\ttfamily},%
stringstyle=\color{black!90},%
numbers=left,%
numberstyle=\scriptsize,%
stepnumber=1,%
numbersep=8pt,%
showstringspaces=false,%
breaklines=true,%
backgroundcolor=\color{gray!10},%
string=[s]{"}{"},%
comment=[l]{:\ "},%
morecomment=[l]{:"},%
literate=%
  *{0}{{{\color{black!70}0}}}{1}%
  {1}{{{\color{black!70}1}}}{1}%
  {2}{{{\color{black!70}2}}}{1}%
  {3}{{{\color{black!70}3}}}{1}%
  {4}{{{\color{black!70}4}}}{1}%
  {5}{{{\color{black!70}5}}}{1}%
  {6}{{{\color{black!70}6}}}{1}%
  {7}{{{\color{black!70}7}}}{1}%
  {8}{{{\color{black!70}8}}}{1}%
  {9}{{{\color{black!70}9}}}{1}}%
\definecolor{pdtblack}{RGB}{0,0,0}
\definecolor{pdtblue}{RGB}{11,93,174}
\definecolor{pdtdarkblue}{RGB}{6,26,64}
\definecolor{pdtdarkgreen}{RGB}{49,92,43}
\definecolor{pdtdarkpurple}{RGB}{50,14,59}
\definecolor{pdtdarkred}{RGB}{61,19,8}
\definecolor{pdtgray}{RGB}{120,120,120}
\definecolor{pdtgreen}{RGB}{100,161,27}
\definecolor{pdtlightblue}{RGB}{59,175,236}
\definecolor{pdtlightgreen}{RGB}{149,198,35}
\definecolor{pdtlightpurple}{RGB}{197,137,232}
\definecolor{pdtlightred}{RGB}{206,62,21}
\definecolor{pdtpurple}{RGB}{106,20,125}
\definecolor{pdtred}{RGB}{145,0,33}
\definecolor{pdtwhite}{RGB}{255,255,255}
\definecolor{pdtyellow}{RGB}{232,163,26}

\begin{document}

\copyrightyear{2026}
\acmYear{2026}
\setcopyright{none}
\acmConference[HSCC '26]{29th ACM International Conference on Hybrid Systems: Computation and Control}{May 11--14, 2026}{Saint Malo, France}
\acmBooktitle{29th ACM International Conference on Hybrid Systems: Computation and Control (HSCC '26), May 11--14, 2026, Saint Malo, France}
\acmDOI{10.1145/3801146.3805668}
\acmISBN{979-8-4007-2566-1/2026/05}

\author{Peng Xie}
\affiliation{%
  \department{Department of Computer Engineering}
  \institution{TUM School of Computation, Information and Technology, Technical University of Munich}
  \city{Heilbronn}
  \country{Germany}}
\email{p.xie@tum.de}

\author{Johannes Betz}
\affiliation{%
  \department{Department of Autonomous Vehicle Systems}
  \institution{TUM School of Engineering and Design, Technical University Munich}
  \city{Garching}
  \country{Germany}}
\email{johannes.betz@tum.de}

\author{Amr Alanwar}
\affiliation{%
  \department{Department of Computer Engineering}
  \institution{TUM School of Computation, Information and Technology, Technical University of Munich}
  \city{Heilbronn}
  \country{Germany}}
\email{alanwar@tum.de}

\begin{abstract}
Optimal path planning in nonconvex free spaces poses substantial computational challenges. A common approach formulates such problems as mixed-integer linear programs (MILPs); however, solving general MILPs is computationally intractable and severely limits scalability. To address these limitations, we propose \textbf{HZ-MP}, an \emph{informed Hybrid Zonotope-based Motion Planner}, which decomposes the obstacle-free space and performs low-dimensional face sampling guided by an ellipsotope heuristic, thereby concentrating exploration on promising transition regions. This structured exploration mitigates the excessive wasted sampling that degrades existing informed planners in narrow-passage or enclosed-goal scenarios. We prove that HZ-MP is \emph{probabilistically complete} and \emph{asymptotically optimal}, and demonstrate empirically that it converges to high-quality trajectories within a small number of iterations.
\end{abstract}

\title{Informed Hybrid Zonotope-based Motion Planning Algorithm}

\maketitle

 \section{Introduction}\label{sec:intro}
Motion planning in robotics requires finding collision-free paths through complex environments. A common approach encodes obstacle avoidance using mixed-integer programs (MIPs), which enable optimal trajectory generation under dynamic and kinematic constraints. However, solving the resulting mixed-integer linear or quadratic programs remains computationally hard~\cite{Ioan2021}. Recent advances have aimed to mitigate this complexity. For example, the hybrid zonotope representation of obstacle-free space has been introduced to compactly model nonconvex regions with a combination of continuous and binary variables~\cite{Bird2023}. By using hybrid zonotopes to encode obstacles within an MPC framework, one can tighten convex relaxations and exploit problem structure, yielding order-of-magnitude speed-ups in mixed-integer solver performance~\cite{Robbins2024}. Other works develop multi-stage Mixed-Integer Quadratic Programming (MIQP) methods that replace direct nonconvex computations with structured decompositions, employing tailored branch-and-bound algorithms integrated with interior-point solvers to accelerate the optimization process~\cite{Robbins2024}. However, MIQP methods still face fundamental limitations in complex scenarios. In high-dimensional settings, MIQP suffers from the curse of dimensionality: an explosion in the number of branch-and-bound nodes caused by many integer variables, large constraint matrices, and dramatically increased solution times. Furthermore, the branch-and-bound procedure offers limited predictability—it may find solutions quickly or require exploring exponentially many branches—making it difficult to provide strong real-time guarantees~\cite{shoja2022overall}. These inherent limitations of optimization-based approaches motivate the exploration of alternative paradigms for efficient motion planning.

Another class of motion planners relies on random sampling to explore the state space. Early algorithms such as the probabilistic roadmap (PRM) and rapidly-exploring random tree (RRT) families are \emph{probabilistically complete} but not optimal by design~\cite{Kavraki1996,Karaman2011}. The \emph{informed} variants bias samples toward cost-bounded ellipsoidal subsets to accelerate convergence; e.g., Informed RRT* improves substantially over uniform RRT*~\cite{gammell_informed_2014}. Notably, adaptively informed trees and effort informed trees employ an asymmetric bidirectional search in which two growing trees continuously guide each other using updated cost-to-go estimates~\cite{strub_ijrr22}. However, stochastic exploration can stall in narrow-passage or double-enclosure scenarios, where a large proportion of samples falls outside reachable regions~\cite{orthey2021section,szkandera2020narrow}, significantly slowing convergence toward an optimal solution~\cite{strub_ijrr22}.


\begin{figure*}[h!]
    \centering
\includegraphics[width=\textwidth]{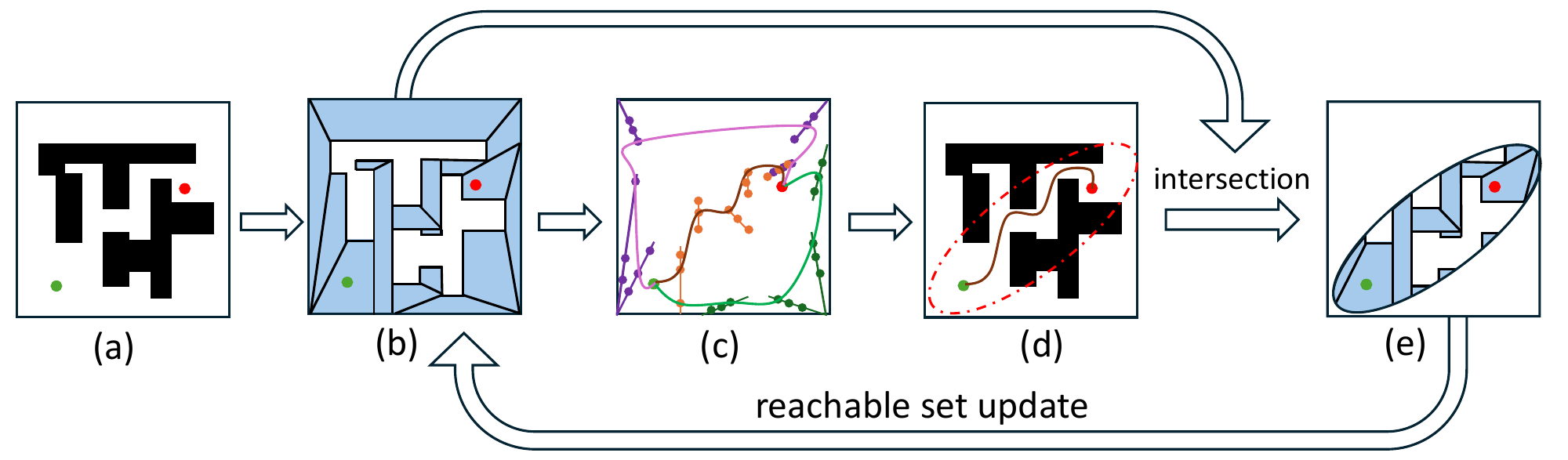}
    \caption{Overview of HZ-MP process illustration. 
    (a) Original environment with obstacles (black), start (green), and goal (red). (b) Hybrid zonotope representation of obstacle-free space decomposed into convex leaves. (c) Three possible connected paths identified through adjacency computation, with the brown path being the shortest. (d) Ellipsotope reachable set constructed based on the optimal path cost from (c). (e) Updated reachable set after pruning using ellipsotope-based bounds, which refines the feasible space for subsequent planning iterations.}
    \label{fig:overview}
\end{figure*}

We propose HZ-MP (informed Hybrid Zonotope-based Motion Planner), which integrates optimization-based and sampling-based planning through four key steps illustrated in Fig.~\ref{fig:overview}.

\emph{(a)→(b) Space decomposition:} The obstacle-free space is represented with hybrid zonotopes—finite collections of convex leaves with discrete transitions. The hybrid zonotope representation and the associated free-space decomposition are based on prior work~\cite{Bird2023,koeln2024zonolab}.

\emph{(b)→(c) Path identification (novel contribution):} We derive a linear feasibility system (Proposition~3.2) that characterizes leaf adjacency and shared-face intersections, guaranteeing collision-free transitions between convex regions. The algorithm samples on $(n-1)$-dimensional shared faces rather than $n$-dimensional volumes and performs parallel computation across all candidate paths simultaneously. In Fig.~\ref{fig:overview}(c), three paths of different colors represent possible routes, with the brown path identified as currently optimal.

\emph{(c)→(d) Ellipsotope construction (novel contribution):} Using the current optimal path cost, an ellipsotope reachable set is constructed to bound all states that could potentially yield improved solutions, as shown in Fig.~\ref{fig:overview}(d). We prove (Lemma~3.5) that this pruning step never discards the optimal path.

\emph{(d)→(e) Reachable set refinement:} Search regions are pruned using ellipsotope-based reachability bounds to update the feasible space, creating a refined reachable set for subsequent planning iterations, as demonstrated in Fig.~\ref{fig:overview}(e).

This iterative process enables HZ-MP to systematically reduce the search space while maintaining solution completeness, effectively balancing global optimality with computational efficiency.

By combining these techniques, HZ-MP achieves a favorable balance between the global solution quality of optimization-based methods and the adaptivity and speed of informed sampling. The proposed method is probabilistically complete and asymptotically optimal, bringing motion planning closer to achieving the solution quality of optimization-based methods while operating within the strict timing constraints of real-time applications.

The remainder of this paper is organized as follows. Section~\ref{sec-prelim} introduces the mathematical preliminaries on hybrid zonotopes and ellipsotopes. Section~\ref{sec-mainpart} presents the HZ-MP algorithm together with its theoretical guarantees. Section~\ref{sec-numericalEx} demonstrates the effectiveness of the algorithm through numerical examples. Section~\ref{sec-conclusion} concludes the paper with directions for future research.
\section{Notations and preliminaries}\label{sec-prelim}

This section introduces the notation and mathematical foundations required for the proposed hybrid zonotope-based motion planning algorithm.

\subsection{Notations}
Matrices are denoted by uppercase letters, e.g., $G\in\mathbb{R}^{n\times n_g}$, and sets by uppercase calligraphic letters, e.g., $\mathcal{Z}\subset\mathbb{R}^{n}$. Vectors and scalars are denoted by lowercase letters, e.g., $b\in\mathbb{R}^{n_c}$.
The $n$-dimensional unit hypercube is denoted by $\mathcal{B}_{\infty}^n=\left\{x\in\mathbb{R}^{n}~\middle|~\|x\|_{\infty}\leq1\right\}$. The set of all $n$-dimensional binary vectors is denoted by $\{-1,1\}^{n}$. The cardinality of a discrete set $\mathcal{T}$ is denoted by $|\mathcal{T}|$; for example, $|\mathcal{T}|=8$ for $\mathcal{T}=\{-1,1\}^{3}$.
The concatenation of two column vectors into a single column vector is denoted by $(\xi_1~\xi_2)=[\xi_1^T~\xi_2^T]^T$. The bold symbols $\mathbf{1}$ and $\mathbf{0}$ denote matrices of all ones and all zeros, respectively, and $\mathbf{I}$ denotes the identity matrix, with dimensions indicated by subscripts when not clear from context.
Given the sets $\mathcal{Z},\:\mathcal{W}\subset\mathbb{R}^{n}$ and $\mathcal{Y}\subset\mathbb{R}^{m}$, and a matrix $R\in\mathbb{R}^{m\times n}$, the linear mapping of $\mathcal{Z}$ by $R$ is $R\mathcal{Z}=\{Rz~|~z\in\mathcal{Z}\}$, the Minkowski sum of $\mathcal{Z}$ and $\mathcal{W}$ is $\mathcal{Z}\oplus\mathcal{W}=\{z+w~|~z\in\mathcal{Z},\:w\in\mathcal{W}\}$, the generalized intersection of $\mathcal{Z}$ and $\mathcal{Y}$ under $R$ is $\mathcal{Z}\cap_R\mathcal{Y}=\{z\in\mathcal{Z}~|~Rz\in\mathcal{Y}\}$, and the union of $\mathcal{Z}$ and $\mathcal{W}$ is $\mathcal{Z}\cup\mathcal{W}=\{x\in\mathbb{R}^{n}~|~x\in\mathcal{Z}\lor x\in\mathcal{W}\}$. The notation $\|.\|$ denotes the Euclidean norm.

\subsection{Hybrid zonotope}\label{sec-hybridZonotopes}

Hybrid zonotopes provide a powerful representation for nonconvex sets by combining continuous and binary factors. Each combination of binary factors defines a constrained zonotope, referred to as a leaf of the hybrid zonotope. Understanding the adjacency relationships between these leaves is essential for applications such as motion planning and reachability analysis. We begin by defining the hybrid zonotope set representation.

\begin{definition}[\cite{Bird2023}]\label{def-hybridZono}
The set $\mathcal{Z}_h\subset\mathbb{R}^n$ is a hybrid zonotope if there exist $G^c\in\mathbb{R}^{n\times n_{g}}$, $G^b\in\mathbb{R}^{n\times n_{b}}$, $c\in\mathbb{R}^{n}$, $A^c\in\mathbb{R}^{n_{c}\times n_{g}}$, $A^b\in\mathbb{R}^{n_{c}\times n_{b}}$, and $b\in\mathbb{R}^{n_c}$ such that
\begin{equation}\label{def-eqn-hybridZono}
    \mathcal{Z}_h = \left\{ \left[G^c \: G^b\right]\left[\begin{smallmatrix}\xi^c \\ \xi^b \end{smallmatrix}\right]  + c\: \middle| \begin{matrix} \left[\begin{smallmatrix}\xi^c \\ \xi^b \end{smallmatrix}\right]\in \mathcal{B}_\infty^{n_{g}} \times \{-1,1\}^{n_{b}}, \\ \left[A^c \: A^b\right]\left[\begin{smallmatrix}\xi^c \\ \xi^b \end{smallmatrix}\right] = b \end{matrix} \right\}.
\end{equation}
where $n_g$ is the number of continuous generators, $n_b$ is the number of binary factors, and $n_c$ is the number of linear equality constraints.
\end{definition}

 Central to our approach is the decomposition of complex obstacle-free environments into manageable convex regions, each represented as a leaf of a hybrid zonotope.
 \begin{example}
The map with obstacles shown in Figure~\ref{fig:hybrid_zonotope_map} is represented by a hybrid zonotope. Consider the blue region with a central obstacle (white area). This nonconvex feasible region is given by:
\begin{equation}
\mathcal{HZ} = \left\langle G^c, G^b, c, A^c, A^b, b \right\rangle
\end{equation}
where $c = [-41.25,\, 2.5]^T \in \mathbb{R}^2$ is the center,
$G^c \in \mathbb{R}^{2 \times 28}$ is the continuous generator matrix
($n_g = 28$ columns arising from the vertex structure of the convex
decomposition), and $G^b = \mathcal{O}_{2 \times 4}$, indicating that
all binary modes share the same center. The constraint matrices
$A^c \in \mathbb{R}^{16 \times 28}$,
$A^b \in \mathbb{R}^{16 \times 4}$, and $b \in \mathbb{R}^{16}$ encode
the logical constraints that exclude the obstacle region ($n_c = 16$
constraints arising from the boundary halfspaces of the obstacle).
\end{example}

\begin{figure}[h]
    \centering
    \includegraphics[width=0.4\textwidth]{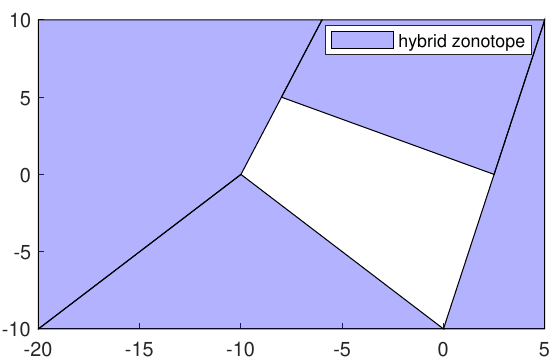}
    \caption{A hybrid zonotope representation of a non-convex feasible region with an obstacle (white area).}
    \label{fig:hybrid_zonotope_map}
\end{figure}
 
 
\subsection{Informed sampling}
The optimal path planning problem seeks to find a path $\sigma^*$ from a start state $x_{\text{start}}$ to a goal region $X_{\text{goal}}$ that minimizes a cost function $c(\sigma)$~\cite{Karaman2011}. Sampling-based planning algorithms are considered almost-surely asymptotically optimal if, as the number of samples approaches infinity, the cost of the solution converges to the optimal cost with probability one.

Informed sampling techniques improve the performance of these algorithms by concentrating the sampling process on regions that are more likely to contain high-quality solutions. Informed RRT*~\cite{gammell_informed_2014} exemplifies this approach by restricting sampling to an ellipsoidal subset of the state space once an initial solution has been found. This informed subset, defined as
\begin{equation}
X_{\text{informed}} = \{x \in X \mid \|x_{\text{start}} - x\| + \|x - x_{\text{goal}}\| \leq c_{\text{best}}\},
\end{equation}
contains all states that could potentially improve the current best solution with cost $c_{\text{best}}$.

A significant theoretical advantage of informed sampling is its convergence rate:

\begin{theorem}[\cite{gammell_informed_2014}]\label{thm:converge}
With uniform sampling of the informed subset, $x \sim \text{Uniform}(X_{\text{informed}})$, the cost of the best solution, $c_{\text{best}}$, converges linearly to the theoretical minimum, $c_{\text{min}}$, in the absence of obstacles.  
\end{theorem}

This linear convergence rate offers a substantial improvement over the sublinear convergence of uniform sampling over the entire state space. Recent advances in bidirectional informed sampling~\cite{strub_ijrr22} further build on this foundation to achieve even more efficient path planning.

 \subsection{Ellipsoids and ellipsotopes}
 To represent the informed region described above, we employ ellipsotopes, which generalize ellipsoids and provide an efficient framework for capturing feasible subsets of the state space.

\begin{definition}[Ellipsoids and Ellipsotopes \cite{kousik2023}]\label{def:ellipsotope}
An ellipsoid is the set
\begin{align}\label{eq:ellipsoid}
    \mathcal{E}(c,Q) = \left\{x \in \mathbb{R}^n\ |\ (x-c)^\top Q(x-c) \leq 1 \right\},
\end{align}
where $c \in \mathbb{R}^n$ is the center and $Q \succ 0$ is a positive definite shape matrix. Generalizing the ellipsoid concept, an ellipsotope is a set
\begin{align}\begin{split}\label{eq:elli_defn}
    \mathcal{E}_p(c,G,A,b,\mathcal{I}) = \big\{c+ G\xi\ |\ &\|\xi_J\|_p \leq 1\ \forall\ J \in \mathcal{I}\\
        &\text{and}\ A\xi = b \big\} \subset \mathbb{R}^n,
\end{split}\end{align}
where $c \in \mathbb{R}^n$, $G \in \mathbb{R}^{n\times n_g}$, $A \in \mathbb{R}^{n_c\times n_g}$, $b \in \mathbb{R}^{n_c}$, and $\mathcal{I}$ is a valid index set.
\end{definition}

\section{Informed hybrid zonotope-based motion planning with space decomposition}\label{sec-mainpart}

\subsection{Problem Formulation and Overview}

Let $\mathcal{X} \subset \mathbb{R}^n$ be the state space, $\mathcal{X}_{\text{obs}} \subset \mathcal{X}$ the obstacle set, and $\mathcal{X}_{\text{free}} = \mathcal{X} \setminus \mathcal{X}_{\text{obs}}$ the obstacle-free space. The path cost is the arc length $c(\sigma) = \int_0^1 \|\dot{\sigma}(t)\| \,dt$.

The optimal motion planning problem is stated as follows: given a start state $x_{\text{start}} \in \mathcal{X}_{\text{free}}$ and a goal state $x_{\text{goal}} \in \mathcal{X}_{\text{free}}$, find a continuous collision-free path $\sigma: [0,1] \rightarrow \mathcal{X}_{\text{free}}$ that minimizes the path cost:
\begin{equation}
\sigma^* = \arg\min_{\sigma} c(\sigma), \quad \text{s.t.} \quad \sigma(0) = x_{\text{start}}, \; \sigma(1) = x_{\text{goal}}
\end{equation}

The HZ-MP algorithm addresses this problem through the systematic process illustrated in Figure~\ref{fig:overview}, transforming the nonconvex planning problem into a structured exploration of convex regions to enable efficient path optimization while maintaining completeness guarantees. The remainder of this section details each component of this process.

\subsection{Free space decomposition via hybrid zonotopes}\label{freespacedecom}

The first step of the approach, illustrated in Figure~\ref{fig:overview}(a)-(b), transforms the obstacle-laden environment into a structured representation amenable to efficient exploration. Given $\mathcal{X}_{\text{free}}$ as defined above, we decompose it into a hybrid zonotope representation through two stages.

\textit{Nef polyhedron representation.} We employ CGAL's Nef polyhedron framework to compute $\mathcal{X}_{\text{free}} = \mathcal{X} \setminus \mathcal{X}_{\text{obs}}$ through exact Boolean operations. This approach guarantees topologically consistent results even for complex non-manifold geometries with holes or narrow passages, while preserving the face-edge incidence relationships essential for adjacency computation~\cite{hachenberger2007boolean}.

\textit{Convex decomposition.} The resulting Nef polyhedron is partitioned into convex components $\{C_1, C_2, \ldots, C_m\}$ such that $\mathcal{X}_{\text{free}} = \bigcup_{i=1}^m C_i$, yielding a finite collection of convex regions suitable for hybrid zonotope construction~\cite{cgal:eb-24b}.

The resulting decomposition is represented by a vertex matrix $V=[\mathbf{v}_1,\dots,\mathbf{v}_{n_v}]\in\mathbb{R}^{n\times n_v}$ containing all vertices of the convex components, together with an incidence matrix $M\in\mathbb{R}^{n_v \times n_F}$~\cite{robbins2024mixed}. From this V-representation, a hybrid zonotope is constructed following the methodology of~\cite{koeln2024zonolab}, which provides efficient conversion from vertex-representation polytopes to a unified hybrid zonotope. This transformation preserves the topology of the free space while yielding a representation amenable to sampling-based motion planning.

The hybrid zonotope representation provides a finite collection of convex regions $\mathcal{L} = \{\mathcal{Z}_1, \ldots, \mathcal{Z}_m\}$, transforming the continuous planning problem into a discrete graph search augmented with continuous optimization within each region. To leverage this structure, we next establish connectivity relationships between regions.

\subsection{Adjacency computation and path identification}

Having decomposed the free space into convex leaves as shown in Figure~\ref{fig:overview}(b), we next establish connectivity relationships to identify feasible paths. This adjacency information, visualized in Figure~\ref{fig:overview}(c), forms the foundation for the dimension-reduced sampling strategy.

In the hybrid zonotope representation of obstacle-free environments, the nonconvex space decomposes into simpler convex regions termed leaves, each corresponding to a specific configuration of binary variables. The key observation for path planning is that adjacent regions share boundaries across which paths can transition smoothly. This adjacency structure yields a structured roadmap for navigation, enabling focused sampling on region boundaries rather than uniform volumetric exploration, while guaranteeing collision-free paths within each convex region.

\begin{remark}\label{prop-adjacency}
Given a hybrid zonotope $\mathcal{Z}_h$ as in Definition~\ref{def-hybridZono}, it can be decomposed into a collection of constrained zonotopes whose connectivity is established as follows.

\begin{enumerate}[(i)]
    \item Based on~\cite[Theorem~5]{Bird2023}, let $\xi_i^b$ be an element of the discrete set $\{-1,1\}^{n_b}$, which contains $2^{n_b}$ elements. The set of feasible binary configurations $\mathcal{L}$ is defined as:
    \begin{equation}
    \mathcal{L} = \{\xi_i^b \in \{-1,1\}^{n_b} \mid \exists \xi^c \in \mathcal{B}_\infty^{n_g} \text{ s.t. } A^c\xi^c + A^b\xi_i^b = b\}
    \end{equation}

    \item For each $\xi_i^b \in \mathcal{L}$, the corresponding constrained zonotope is:
    \begin{equation}
    \mathcal{Z}_{c,i} = \langle G^c, c + G^b\xi_i^b, A^c, b - A^b\xi_i^b \rangle \label{leaf_exp}
    \end{equation}

    \item An indexing function $\text{ID}: \mathcal{L} \rightarrow \{1,2,\ldots,|\mathcal{L}|\}$ assigns a unique integer to each leaf, enabling the adjacency matrix $\mathcal{A} \in \{0,1\}^{|\mathcal{L}| \times |\mathcal{L}|}$ to be defined as:
    \begin{equation}
    \mathcal{A}_{ij} =
    \begin{cases}
    1 & \text{if } \mathcal{Z}_{c,i} \cap \mathcal{Z}_{c,j} \neq \emptyset \\
    0 & \text{otherwise}
    \end{cases}
    \end{equation}
    where $i = \text{ID}(\xi^b_i)$ and $j = \text{ID}(\xi^b_j)$ for $\xi^b_i, \xi^b_j \in \mathcal{L}$.
\end{enumerate}

This decomposition and the resulting adjacency structure provide a complete topological representation of the hybrid zonotope that can be systematically explored for motion planning.
\end{remark}

\begin{proposition}\label{prop:leaf-intersection}
Let $\mathcal{Z}_h=\langle G^{c},G^{b},c,A^{c},A^{b},b\rangle$
be a hybrid zonotope with leaves $\mathcal{Z}_{c,i}$ and $\mathcal{Z}_{c,j}$ as given in~(\ref{leaf_exp}), and define
\begin{equation}
\begin{aligned}
M&=\begin{bmatrix}G^{c}\\A^{c}\end{bmatrix},\;
N=\begin{bmatrix}G^{b}\\A^{b}\end{bmatrix},\;\xi^c = \xi^c_i,\;
\\
\Delta\xi^c &= \xi^c_j - \xi^c_i,\;
\Delta\xi^{b}=\xi^{b}_{i}-\xi^{b}_{j},\;
r_{i}=b-A^{b}\xi^{b}_{i}.
\end{aligned}
\end{equation}
\textit{(a) Intersection criterion.}
$\mathcal{Z}_{c,i} \cap \mathcal{Z}_{c,j} \neq \emptyset$ if and only if the following system is feasible:
\begin{subequations}\label{eq:leaf-system}
\begin{align}
M\Delta\xi^c &= N\Delta\xi^b \label{eq:leaf-system:diff}\\
A^c\xi^c &= r_i \label{eq:leaf-system:ac}\\
-\mathbf{1} \leq \xi^c &\leq \mathbf{1} \label{eq:leaf-system:box1}\\
-\mathbf{1} \leq \xi^c + \Delta\xi^c &\leq \mathbf{1} \label{eq:leaf-system:box2}
\end{align}
\end{subequations}
\textit{(b) Linear consistency shortcut.}
Let $L = (I-MM^\dagger)N$ where $M^\dagger$ is the Moore–Penrose pseudoinverse of $M$. Then
\begin{align}
\exists \Delta\xi^c: M\Delta\xi^c = N\Delta\xi^b \iff L\Delta\xi^b = 0
\end{align}
\textit{(c) Tangent vs. overlap characterization.}
Let $\mathcal{F}$ denote the set of all solutions to \eqref{eq:leaf-system}, and define the strict interior feasible region:
\begin{align}
\mathcal{F}^\circ = \{(\xi^c, \Delta\xi^c) \in \mathcal{F} \mid |\xi^c_k| < 1, |\xi^c_k + \Delta\xi^c_k| < 1 ; \forall k\}
\end{align}
Then:
\begin{align}
\text{Tangent contact (boundary-only)} &\iff \mathcal{F} \neq \emptyset \text{ and } \mathcal{F}^\circ = \emptyset \\
\text{Interior overlap} &\iff \mathcal{F}^\circ \neq \emptyset
\end{align}
Interior overlap implies the intersection contains points in the interior of both leaves, yielding non-zero volume. Tangent contact implies the intersection is confined to boundaries, forming a lower-dimensional face.\\
\textit{(d) Adjacency computation.}
For any pair of leaves, the adjacency can be determined by:
\begin{align}
\mathcal{A}_{ij} =
\begin{cases}
1 & \text{if } L(\xi^b_i - \xi^b_j) = 0 \text{ and system \eqref{eq:leaf-system} is feasible} \\
0 & \text{otherwise}
\end{cases}
\end{align}
\end{proposition}

\begin{proof}
\textit{(a)} ($\Rightarrow$) Let $z \in \mathcal{Z}_{c,i} \cap \mathcal{Z}_{c,j}$. Then there exist $\xi^c_i, \xi^c_j \in \mathcal{B}_\infty^{n_g} = \{\xi \in \mathbb{R}^{n_g} \mid \|\xi\|_\infty \leq 1\}$ such that:
\begin{align}
G^c\xi^c_i + c + G^b\xi^b_i &= z \\
G^c\xi^c_j + c + G^b\xi^b_j &= z \\
A^c\xi^c_i + A^b\xi^b_i &= b \\
A^c\xi^c_j + A^b\xi^b_j &= b
\end{align}
Setting $\Delta\xi^c = \xi^c_j - \xi^c_i$ and subtracting yields:
\begin{align}
G^c\Delta\xi^c &= G^b(\xi^b_i - \xi^b_j) = G^b\Delta\xi^b\\
A^c\Delta\xi^c &= A^b(\xi^b_i - \xi^b_j) = A^b\Delta\xi^b
\end{align}
Combined, this gives $M\Delta\xi^c = N\Delta\xi^b$.
Let $\xi^c = \xi^c_i$. Then we have:
\begin{align}
A^c\xi^c &= A^c\xi^c_i = b - A^b\xi^b_i = r_i\\
\|\xi^c\|_\infty &= \|\xi^c_i\|_\infty \leq 1\\
\|\xi^c + \Delta\xi^c\|_\infty &= \|\xi^c_i + (\xi^c_j - \xi^c_i)\|_\infty = \|\xi^c_j\|_\infty \leq 1
\end{align}
Thus, $(\xi^c, \Delta\xi^c)$ satisfies system \eqref{eq:leaf-system}.

($\Leftarrow$) Given $\xi^c$ and $\Delta\xi^c$ satisfying \eqref{eq:leaf-system}, define:
\begin{align}
\xi^c_i &= \xi^c\\
\xi^c_j &= \xi^c + \Delta\xi^c
\end{align}
From \eqref{eq:leaf-system:box1} and \eqref{eq:leaf-system:box2}, we have $\xi^c_i, \xi^c_j \in \mathcal{B}_\infty^{n_g}$. From \eqref{eq:leaf-system:ac}, $A^c\xi^c_i = r_i = b - A^b\xi^b_i$, thus $A^c\xi^c_i + A^b\xi^b_i = b$. From \eqref{eq:leaf-system:diff}, $M\Delta\xi^c = N\Delta\xi^b$ implies:
\begin{align}
A^c(\xi^c_j - \xi^c_i) &= A^b(\xi^b_i - \xi^b_j)\\
\Rightarrow A^c\xi^c_j + A^b\xi^b_j &= A^c\xi^c_i + A^b\xi^b_i = b
\end{align}
Define $z = G^c\xi^c_i + c + G^b\xi^b_i$. Similarly, from \eqref{eq:leaf-system:diff}:
\begin{align}
G^c(\xi^c_j - \xi^c_i) &= G^b(\xi^b_i - \xi^b_j)\\
\Rightarrow G^c\xi^c_j + G^b\xi^b_j &= G^c\xi^c_i + G^b\xi^b_i
\end{align}
Thus, $z = G^c\xi^c_j + c + G^b\xi^b_j \in \mathcal{Z}_{c,i} \cap \mathcal{Z}_{c,j}$.

\textit{(b)} $N\Delta\xi^b$ lies in $\operatorname{im}M$ if and only if its projection onto $\operatorname{im}M^\perp$ is zero; this projection is precisely $(I-MM^\dagger)N\Delta\xi^b = L\Delta\xi^b$.\\

\textit{(c)} $\mathcal{F}^\circ \neq \emptyset$ means there exists a solution strictly inside both boxes, implying the intersection contains interior points with non-zero volume. When all feasible solutions touch at least one box face ($|\cdot| = 1$), the intersection consists only of boundary points or lower-dimensional faces, resulting in tangent contact.\\

\textit{(d)} Adjacency requires both linear consistency ($L\Delta\xi^b = 0$) and feasibility of system \eqref{eq:leaf-system}, which together ensure a non-empty intersection between leaves.
\end{proof}


\begin{corollary}\label{cor-boundary-contact}
To distinguish between interior overlap and boundary contact of leaves $\mathcal{Z}_{c,i}$ and $\mathcal{Z}_{c,j}$, introduce a slack variable $\delta \geq 0$ that tightens the box constraints in \eqref{eq:leaf-system}:
\begin{subequations}\label{eq:leaf-slack}
\begin{align}
\eqref{eq:leaf-system:diff}\\
\eqref{eq:leaf-system:ac}\\
-1+\delta \leq \xi^c \leq 1-\delta \label{eq:leaf-slack:box1}\\
-1+\delta \leq \xi^c + \Delta\xi^c \leq 1-\delta \label{eq:leaf-slack:box2}
\end{align}
\end{subequations}
Consider the linear program
\begin{align}
\delta^\star = \max_{\xi^c, \Delta\xi^c, \delta} \delta \quad \text{s.t. }~\eqref{eq:leaf-slack}
\end{align}
Then:
\begin{itemize}
\item $\delta^\star > 0 \iff$ interior overlap (intersection contains interior points)
\item $\delta^\star = 0 \iff$ boundary contact (intersection confined to boundaries)
\end{itemize}
Intuitively, $\delta^\star$ quantifies the maximal uniform contraction of the box constraints that preserves feasibility.
\end{corollary}

\begin{corollary}\label{cor-sharedFaces}
For any two adjacent leaf nodes $\mathcal{Z}_{c,i}$ and $\mathcal{Z}_{c,j}$ with $\mathcal{A}_{ij} = 1$, their shared face can be computed as a generalized intersection of constrained zonotopes~\cite{scott_constrained_2016}. Given $\mathcal{Z}_{c,i} = \langle G^c, c + G^b\xi^b_i, A^c, b - A^b\xi^b_i \rangle$ and $\mathcal{Z}_{c,j} = \langle G^c, c + G^b\xi^b_j, A^c, b - A^b\xi^b_j \rangle$, their shared face is:
\begin{equation}
\begin{aligned}
\mathcal{F}_{ij} &= \mathcal{Z}_{c,i} \cap_{\mathbf{I}} \mathcal{Z}_{c,j} = \left\{\left[G^c \: \mathbf{0}\right], c + G^b\xi^b_i, \right. \\
&\quad \left. \left[\begin{array}{cc}
A^c & \mathbf{0} \\
\mathbf{0} & A^c \\
G^c & -G^c
\end{array}\right], \left[\begin{array}{c}
b - A^b\xi^b_i \\
b - A^b\xi^b_j \\
G^b(\xi^b_i - \xi^b_j)
\end{array}\right]\right\}
\end{aligned}
\end{equation}
where $\mathbf{I}$ is the identity matrix.

Furthermore, the shared face $\mathcal{F}_{ij}$ is a constrained zonotope of dimension at most $(n-1)$. This follows from the fact that the constraint $G^c\xi = G^b(\xi^b_i - \xi^b_j)$ imposes at least one linear constraint on $\xi$, reducing the dimension by at least 1. Since $\xi^b_i \neq \xi^b_j$ for adjacent leaves (they are distinct), the constraint is non-trivial and $\mathcal{F}_{ij}$ has positive $(n-1)$-dimensional measure.
\end{corollary}

This computation of shared faces is based on the generalized intersection operation for constrained zonotopes established in~\cite{scott_constrained_2016}. The dimension-reduction property is central to the proposed sampling-based approach: by focusing samples on these $(n-1)$-dimensional shared faces rather than $n$-dimensional volumes, the dimensionality of the search space is effectively reduced while connectivity between adjacent regions of the free space is preserved.
\begin{example}
Consider a hybrid zonotope $\mathcal{Z}_h$ in $\mathbb{R}^3$ that decomposes a workspace into four distinct leaf nodes representing convex regions of the free space.
Figure~\ref{fig:hybrid_zonotope_example} shows the geometric representation of these leaf nodes.

\begin{figure}[t]
\centering
\includegraphics[width=0.7\linewidth]{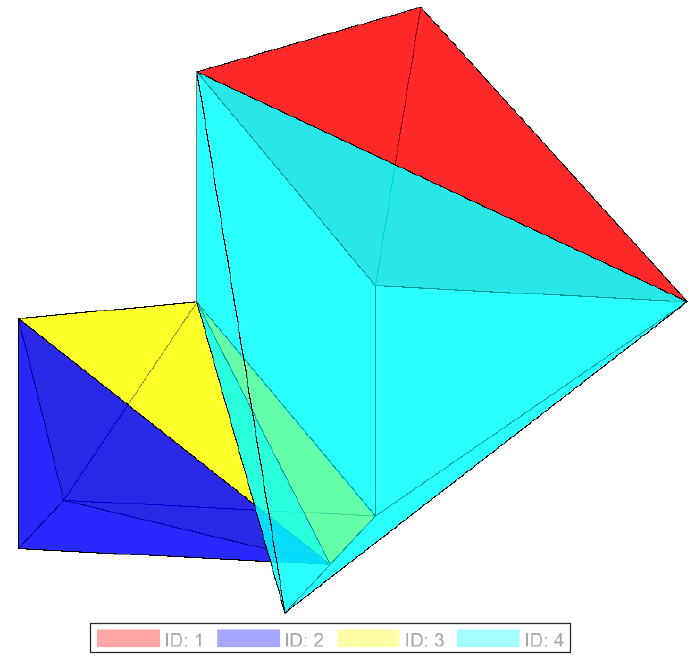}
\caption{Hybrid zonotope with four leaf nodes (1-4) and their adjacency.}
\label{fig:hybrid_zonotope_example}
\end{figure}

The adjacency matrix $\mathcal{A}$ for this configuration is:
\begin{equation}
\mathcal{A} = \begin{array}{c|cccc}
\text{ID} & 1 & 2 & 3 & 4 \\
\hline
1 & 0 & 0 & 0 & 1 \\
2 & 0 & 0 & 1 & 0 \\
3 & 0 & 1 & 0 & 1 \\
4 & 1 & 0 & 1 & 0
\end{array}
\end{equation}

This adjacency matrix indicates that leaf node 1 is adjacent to leaf node 4, leaf node 4 is adjacent to leaf node 3, and leaf node 2 is adjacent to leaf node 3, the only possible route is $1 \rightarrow 4 \rightarrow 3 \rightarrow 2$. 
For any pair of adjacent leaf nodes, we can compute their shared face using Corollary \ref{cor-sharedFaces}. For example, the shared face between leaf nodes 1 and 4 is obtained by:
\begin{equation}
\mathcal{F}_{1,4} = \mathcal{Z}_{c,1} \cap_{\mathbf{I}} \mathcal{Z}_{c,4}
\end{equation}

Our algorithm generates samples on the shared faces $\mathcal{F}_{1,4}$, $\mathcal{F}_{4,3}$, and $\mathcal{F}_{3,2}$ to optimize the path through these waypoints, achieving dimension reduction by sampling on lower-dimensional interfaces rather than the full 3D space.
\end{example}

The adjacency matrix $\mathcal{A}$ and the shared-face computations of Corollary~\ref{cor-sharedFaces} provide the topological foundation for the planning algorithm. We now present how this structure enables efficient path optimization through informed sampling and iterative refinement.

\subsection{Informed sampling and ellipsotope refinement}

With the adjacency structure in place, we employ an iterative refinement process that alternates between path optimization and reachable-set updates, as depicted in Figure~\ref{fig:overview}(c)-(e). The following subsections describe the sampling strategy on shared faces and the ellipsotope-based pruning mechanism.

\subsubsection{Sampling on shared faces}

Rather than sampling uniformly throughout the $n$-dimensional state space, we exploit the structure of Proposition~\ref{prop:leaf-intersection} to focus samples on the $(n-1)$-dimensional shared faces $\mathcal{F}_{ij}$ between adjacent leaves. For a path $p = (i_0, i_1, \ldots, i_k)$ through the adjacency graph, waypoint samples $\mathbf{s} = (s_1, \ldots, s_{k-1})$ are generated with each $s_j \in \mathcal{F}_{i_j,i_{j+1}}$.

The path cost for a given set of waypoints is:
\begin{equation}
c(\sigma_{\mathbf{s}}) = \|x_{\text{start}} - s_1\| + \sum_{j=1}^{k-2} \|s_j - s_{j+1}\| + \|s_{k-1} - x_{\text{goal}}\|
\end{equation}

The convexity of each leaf guarantees that straight-line segments between consecutive waypoints remain collision-free, enabling efficient local optimization within this reduced-dimensional space.

\subsubsection{Ellipsotope-based reachable set refinement}

As illustrated in Figure~\ref{fig:overview}(d)-(e), each discovered path with cost $c_{\text{best}}$ induces an ellipsotope that bounds all potentially improving solutions:
\begin{equation}
\mathcal{E} = \{x \in \mathbb{R}^n : \|x - x_{\text{start}}\| + \|x - x_{\text{goal}}\| \leq c_{\text{best}}\}
\end{equation}

This ellipsotope, shown as the dashed ellipse in Figure~\ref{fig:overview}(d), enables systematic pruning of the search space. The intersection $\mathcal{E} \cap \mathcal{Z}_h$ yields the refined reachable set depicted in Figure~\ref{fig:overview}(e), focusing subsequent iterations on regions that can potentially improve the solution.

\begin{figure}[h]
    \centering
    \begin{subfigure}[t]{0.23\textwidth}
        \centering
        \includegraphics[width=\textwidth]{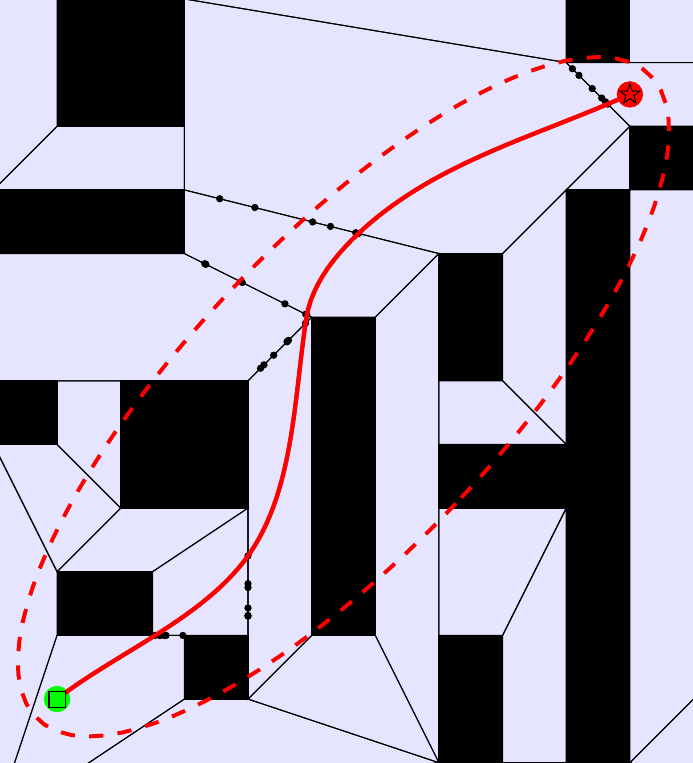}
        \caption{The green dot marks the start state, the red dot marks the goal, and the red curve is the current shortest path. The red dashed ellipse denotes the ellipsotope corresponding to the current best cost.}
        \label{fig:ellipsotope_intersection}
    \end{subfigure}
    \hfill
    \begin{subfigure}[t]{0.23\textwidth}
        \centering        
        \includegraphics[width=\textwidth]{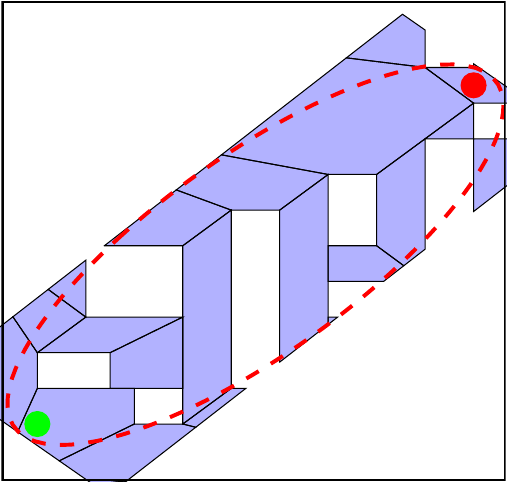}
        \caption{Intersection of the hybrid zonotope with the constrained zonotope derived from the ellipsotope, yielding the updated feasible space for the subsequent planning iteration.}
        \label{fig:updated_path}
    \end{subfigure}
    \caption{Illustration of ellipsotope-informed sampling with path cost.}
    \label{fig:path_optimization}
\end{figure}

Algorithm~\ref{alg:HZMP} implements the complete HZ-MP pipeline illustrated in Figure~\ref{fig:overview}. The procedure comprises three phases: decomposition and adjacency computation (lines 1--7, corresponding to Figure~\ref{fig:overview}(a)--(c)), parallel path optimization (lines 12--20), and ellipsotope-based refinement (lines 17--18, corresponding to Figure~\ref{fig:overview}(d)--(e)).

\begin{algorithm}[t]
\caption{Informed Hybrid Zonotope Motion Planning (HZ-MP)}
\label{alg:HZMP}
\begin{algorithmic}[1]
\Require: Start state $x_{\text{start}}$, goal state $x_{\text{goal}}$, state space $\mathcal{X}$, obstacles $\mathcal{X}_{\text{obs}}$, max iterations $N_{\text{max}}$
\Ensure: Optimal path solution $\sigma^*$
\State $\mathcal{Z}_h \gets \texttt{FreeSpaceDecomposition}(\mathcal{X}, \mathcal{X}_{\text{obs}})$ \Comment{\ref{freespacedecom}}
\State $\mathcal{L} \gets \texttt{LeafNodes}(\mathcal{Z}_h)$
\State $\mathcal{A} \gets \texttt{AdjacencyMatrix}(\mathcal{Z}_h, \mathcal{L})$ \Comment{\ref{prop:leaf-intersection}}
\State $i_{\text{start}} \gets \texttt{findContainingLeaf}(x_{\text{start}}, \mathcal{L})$
\State $i_{\text{goal}} \gets \texttt{findContainingLeaf}(x_{\text{goal}}, \mathcal{L})$
\State $\mathcal{P} \gets \texttt{findAllConnectedPaths}(i_{\text{start}}, i_{\text{goal}}, \mathcal{A})$
\State $\mathcal{F} \gets \texttt{computeSharedFaces}(\mathcal{L}, \mathcal{A})$ \Comment{Corollary~\ref{cor-sharedFaces}}
\State $c_{\text{best}} \gets \infty$
\State $\sigma^* \gets \emptyset$
\State $iter \gets 0$
\State Initialize parallel search processes
\While{$iter < N_{\text{max}}$ \textbf{and} not converged}
  \For{each path $p \in \mathcal{P}$ \textbf{in parallel}}
    \State $S_p \gets \texttt{generateSamples}(\mathcal{F}, p)$
    \State $c_p, \sigma_p \gets \texttt{optimizePath}(x_{\text{start}}, x_{\text{goal}}, S_p, \mathcal{L})$
    \If{$c_p < c_{\text{best}}$}
      \State $c_{\text{best}} \gets c_p$
      \State $\sigma^* \gets \sigma_p$
      \State $(\mathcal{E}, \mathcal{L}, \mathcal{A}, \mathcal{P}) \gets \texttt{updateReachableSet}(x_{\text{start}}, x_{\text{goal}}, c_{\text{best}}, \mathcal{L}, \mathcal{A}, \mathcal{P})$ \Comment{Algo. \ref{alg:updateReachableSet}}
      \State $\mathcal{F} \gets \texttt{computeSharedFaces}(\mathcal{L}, \mathcal{A})$
    \EndIf
  \EndFor
  \State $iter \gets iter + 1$
\EndWhile
\State \textbf{return} $\sigma^*$
\end{algorithmic}
\end{algorithm}

The \texttt{updateReachableSet} procedure (Algorithm~\ref{alg:updateReachableSet}) constructs an ellipsotope from the current best path cost and updates the adjacency relationships to reflect the pruned search space. This iterative refinement concentrates computational resources on regions that admit potential improvement, thereby accelerating convergence to the optimal solution.

\begin{algorithm}[t]
\caption{Update Reachable Set with Ellipsotope}
\label{alg:updateReachableSet}
\begin{algorithmic}[1]
\Require Start $x_{\text{start}}$, goal $x_{\text{goal}}$, current best cost $c_{\text{best}}$, leaves $\mathcal{L}$, adjacency matrix $\mathcal{A}$, path set $\mathcal{P}$
\Ensure Ellipsotope $\mathcal{E}$, pruned leaf set $\mathcal{L}_E$, pruned adjacency $\mathcal{A}_E$, pruned paths $\mathcal{P}_E$
\State $d \gets x_{\text{goal}} - x_{\text{start}}$
\State $\|d\| \gets \sqrt{d^\top d}$
\State $\hat{d} \gets d / \|d\|$
\State $c \gets (x_{\text{start}} + x_{\text{goal}})/2$
\State $a \gets c_{\text{best}}/2$
\State $b \gets \sqrt{a^2 - (\|d\|/2)^2}$
\State $Q \gets \texttt{constructShapeMatrix}(\hat{d}, a, b)$
\State $\mathcal{E} \gets \mathcal{E}_2(c, Q)$ \Comment{Ellipsotope, ~Def.~\ref{def:ellipsotope}; }

\State $\mathcal{L}_E \gets \{\, i \in \mathcal{L} \mid Z^i \cap \mathcal{E} \neq \emptyset \,\}$
\State $\mathcal{A}_E \gets \mathcal{A}[\mathcal{L}_E,\mathcal{L}_E]$
\State $\mathcal{P}_E \gets \emptyset$
\For{each path $\sigma_s$ with leaf indices $(i_1,\dots,i_K)$ in $\mathcal{P}$}
    \If{$i_j \in \mathcal{L}_E$ for all $j=1,\dots,K$}
        \State $\mathcal{P}_E \gets \mathcal{P}_E \cup \{\sigma_s\}$
    \EndIf
\EndFor
\State \textbf{return} $\mathcal{E}, \mathcal{L}_E, \mathcal{A}_E, \mathcal{P}_E$
\end{algorithmic}
\end{algorithm}

Algorithm~\ref{alg:HZMP} operates in three phases. First, it constructs the hybrid zonotope decomposition and identifies the leaves containing the start and goal states (lines 1--7). All feasible paths $\mathcal{P}$ are enumerated via breadth-first search on $\mathcal{A}$. The core optimization phase (lines 12--20) explores each path $p \in \mathcal{P}$ in parallel by drawing samples on the shared faces between consecutive leaves and solving $\mathbf{s}_p^* = \arg\min_{\mathbf{s} \in S_p} c(\sigma_{\mathbf{s}})$ for the best waypoint configuration. After each improvement in $c_{\text{best}}$, Algorithm~\ref{alg:updateReachableSet} constructs an ellipsotope-informed reachable set and prunes leaves, adjacency relations, and candidate paths that lie entirely outside this region (lines 17--18), restricting subsequent sampling to the subset of the graph that can still yield better solutions.

\begin{lemma}[Safe ellipsoidal pruning]\label{lem:ellipsoid-pruning}
Let $c(\sigma)$ denote the path-length cost of any feasible path $\sigma$ from $x_{\mathrm{start}}$ to $x_{\mathrm{goal}}$, let $c_{\mathrm{best}} \geq c^\star$ be the cost of the current best solution with $c^\star$ the optimal cost, and let $\mathcal{E}(c_{\mathrm{best}})$ be the ellipsotope-informed reachable set. Then:
\begin{enumerate}[(i)]
  \item For any feasible path $\sigma$ with $c(\sigma) < c_{\mathrm{best}}$ and any $t \in [0,1]$, it holds that $\sigma(t) \in \mathcal{E}(c_{\mathrm{best}})$.
  \item Suppose $X_{\mathrm{free}} = \bigcup_{i=1}^M Z^i$ is covered by hybrid zonotope leaves, and define
  \[
    \mathcal{L}_\mathcal{E}(c_{\mathrm{best}}) := \{\, i \mid Z^i \cap \mathcal{E}(c_{\mathrm{best}}) \neq \emptyset \,\}.
  \]
  Then any feasible path with cost $c(\sigma) < c_{\mathrm{best}}$ visits only leaves in $\mathcal{L}_\mathcal{E}(c_{\mathrm{best}})$. In particular, pruning all leaves and shared faces contained in $X_{\mathrm{free}} \setminus \mathcal{E}(c_{\mathrm{best}})$ cannot remove the optimal path.
\end{enumerate}
\end{lemma}

\begin{proof}
For any feasible path $\sigma$ and any $t \in [0,1]$, let $x = \sigma(t)$. The subpaths from $x_{\mathrm{start}}$ to $x$ and from $x$ to $x_{\mathrm{goal}}$ have lengths at least $\|x - x_{\mathrm{start}}\|_2$ and $\|x - x_{\mathrm{goal}}\|_2$, respectively. Hence,
\[
  \|x - x_{\mathrm{start}}\| + \|x - x_{\mathrm{goal}}\|
  \leq c(\sigma) < c_{\mathrm{best}},
\]
so $x \in \mathcal{E}(c_{\mathrm{best}})$, proving (i). For (ii), every point $\sigma(t)$ lies in some leaf $Z^{i_t}$ and, by (i), also in $ \mathcal{E}(c_{\mathrm{best}})$, so $Z^{i_t} \cap  \mathcal{E}(c_{\mathrm{best}}) \neq \emptyset$ and $i_t \in \mathcal{L}_ \mathcal{E}(c_{\mathrm{best}})$. Thus any path with cost $< c_{\mathrm{best}}$ only uses leaves in $\mathcal{L}_ \mathcal{E}(c_{\mathrm{best}})$, and its transitions occur on shared faces inside $ \mathcal{E}(c_{\mathrm{best}})$, so pruning outside $ \mathcal{E}(c_{\mathrm{best}})$ cannot discard the optimal path.
\end{proof}

By Lemma~\ref{lem:ellipsoid-pruning}, ellipsoidal pruning never removes the optimal path or any path with cost smaller than $c_{\text{best}}$. Therefore, the pruned HZ-MP algorithm retains the same probabilistic completeness and asymptotic optimality guarantees as the unpruned variant established in Theorems~\ref{thm:completeness} and~\ref{thm:optimality}.

\subsubsection{Path Cost Structure}

For a path defined by waypoints $\mathbf{s} = (s_1, \ldots, s_{k-1})$ with each $s_j \in \mathcal{F}_{i_j,i_{j+1}}$, the total cost is
\begin{equation}
c(\sigma_{\mathbf{s}}) = \|x_{\text{start}} - s_1\| + \sum_{j=1}^{k-2} \|s_j - s_{j+1}\| + \|s_{k-1} - x_{\text{goal}}\|.
\end{equation}
Because each leaf is convex, straight-line segments between consecutive waypoints are guaranteed to be collision-free, thereby eliminating the need for explicit collision checking during path evaluation.

\paragraph{Practical path enumeration.}
Since the number of simple paths can grow exponentially, the implementation caps the maximum path length by a hop limit $L_{\max}$ and retains only the $K$ most promising leaf sequences ranked by a geometric cost lower bound~\cite{yen1971finding, eppstein1998finding}. Combined with ellipsotope pruning (Lemma~\ref{lem:ellipsoid-pruning}), this keeps $|\mathcal{P}|$ moderate in practice.






\subsection{Ellipsotope-informed path refinement}

Upon discovering a path with cost $c_{\text{best}}$, the algorithm constructs an ellipsotope that represents the reachable set:
\begin{equation}
\mathcal{E} = \{x \in \mathbb{R}^n : \|x - x_{\text{start}}\| + \|x - x_{\text{goal}}\| \leq c_{\text{best}}\}.
\end{equation}

This set can be expressed as an ellipsotope (Definition~\ref{def:ellipsotope}) with $p = 2$ and $\mathcal{I} = \{\{1, \ldots, n_g\}\}$:
\begin{equation}
\mathcal{E} = \mathcal{E}_2(c, G) = \{c + G\xi : \|\xi\|_2 \leq 1\},
\end{equation}
where the center $c = \frac{1}{2}(x_{\text{start}} + x_{\text{goal}})$ and the generator matrix $G \in \mathbb{R}^{n \times n}$ is constructed so that the ellipsotope has semi-major axis $a = c_{\text{best}}/2$ along the direction $\hat{d} = \frac{x_{\text{goal}} - x_{\text{start}}}{\|x_{\text{goal}} - x_{\text{start}}\|_2}$ and semi-minor axes $b = \sqrt{a^2 - \|d\|^2/4}$ in the orthogonal directions.

The ellipsotope representation precisely bounds all states that could potentially improve the current solution while enabling efficient intersection operations with constrained zonotope leaves. The intersection of an ellipsotope and a constrained zonotope can be computed using constrained convex generators~\cite{ccg2022}, which provide a systematic method for combining their constraints, as illustrated in Figure~\ref{fig:ellipsotope_intersection}. The algorithm classifies leaf nodes on the basis of their intersection with $\mathcal{E}$: \emph{inactive} if $\mathcal{Z}_{c,i} \cap \mathcal{E} = \emptyset$; \emph{partial} if $\emptyset \neq \mathcal{Z}_{c,i} \cap \mathcal{E} \neq \mathcal{Z}_{c,i}$; and \emph{active} if $\mathcal{Z}_{c,i} \subseteq \mathcal{E}$. The intersection $\mathcal{Z}_{c,i} \cap \mathcal{E}$ can be computed using the techniques of~\cite{holmes2023searchbased}, enabling rapid pruning of infeasible regions. To further simplify the computation, ellipsotopes may also be converted to a constrained zonotope representation, as illustrated in Figure~\ref{fig:updated_path}. This conversion permits all set operations to be performed within the zonotope-based framework, thereby maintaining computational efficiency.

\subsubsection*{Theoretical Guarantees}

The hybrid zonotope framework reduces motion planning in obstacle-cluttered environments to a sequence of convex subproblems, enabling the following theoretical guarantees.



\begin{theorem}[Probabilistic Completeness]\label{thm:completeness}
Let\/ $\mathcal{Z}_h$ be a hybrid zonotope representation of\/ $\mathcal{X}_{\textup{free}}$.
If a feasible solution exists, then
$\lim_{N_s \to \infty} \mathbb{P}[\textup{solution found}] = 1$,
where $N_s$ denotes the number of samples per shared face.
\end{theorem}

\begin{proof}
Let $\sigma^*: [0,1] \to \mathcal{X}_{\text{free}}$ be a feasible path. Because $\mathcal{Z}_h$ represents $\mathcal{X}_{\text{free}}$, there exists a sequence of binary factors $\{\xi^{b,*}_t\}_{t \in [0,1]}$ such that
$$\sigma^*(t) = G^c\xi^{c,*}_t + G^b\xi^{b,*}_t + c$$
for some piecewise continuous factors $\xi^{c,*}_t \in \mathcal{B}_\infty^{n_g}$ satisfying $A^c\xi^{c,*}_t + A^b\xi^{b,*}_t = b$. Note that $\sigma^*$ remains spatially continuous at each switching point $s_j^* \in \mathcal{F}_{i_j,i_{j+1}}$ even though $\xi^{c,*}_t$ may be discontinuous there.

Because $\xi^{b,*}_t \in \{-1,1\}^{n_b}$ is discrete and the path is continuous, there exist times $0 = t_0 < t_1 < \ldots < t_k = 1$ at which $\xi^{b,*}_{t_j} \neq \xi^{b,*}_{t_{j+1}}$. This defines a sequence of leaves $\mathcal{I}^* = (i_0, \ldots, i_k)$ with transition points $s_j^* = \sigma^*(t_j) \in \mathcal{F}_{i_j,i_{j+1}}$.

By Corollary~\ref{cor-sharedFaces}, each shared face $\mathcal{F}_{i_j,i_{j+1}}$ has positive $(n{-}1)$-dimensional measure, i.e., $\mu_{n-1}(\mathcal{F}_{i_j,i_{j+1}}) > 0$.
For uniform sampling on $\mathcal{F}_{i_j,i_{j+1}}$,
$$\mathbb{P}[\|s_j^{(l)} - s_j^*\| < \epsilon] \geq \frac{\mu_{n-1}(B_\epsilon(s_j^*) \cap \mathcal{F}_{i_j,i_{j+1}})}{\mu_{n-1}(\mathcal{F}_{i_j,i_{j+1}})} \geq p_\epsilon > 0.$$

It follows that
$$\mathbb{P}[\min_{l \in \{1,\ldots,N_s\}} \|s_j^{(l)} - s_j^*\| < \epsilon] = 1 - (1 - p_\epsilon)^{N_s} \to 1 \text{ as } N_s \to \infty.$$

Because each leaf $\mathcal{Z}_{c,i}$ is convex (being a constrained zonotope), straight-line paths between points within a leaf are collision-free. The algorithm therefore finds a feasible path with probability approaching~1.
\end{proof}

\begin{theorem}[Asymptotic Optimality]\label{thm:optimality}
Under uniform sampling on the shared faces of $\mathcal{Z}_h$, $\lim_{N_s \to \infty} \mathbb{E}[c_{\textup{best}}] = c^*$, where $c^*$ denotes the optimal cost.
\end{theorem}

\begin{proof}
The optimal path $\sigma^*$ with cost $c^*$ passes through transition points $\mathbf{s}^* = (s_1^*, \ldots, s_{k-1}^*)$. It is first shown that the cost $c(\sigma_{\mathbf{s}})$ is Lipschitz continuous in the waypoint vector $\mathbf{s} = (s_1,\ldots,s_{k-1})$, where
\[
  c(\sigma_{\mathbf{s}}) 
  = \|x_{\text{start}} - s_1\|
    + \sum_{j=1}^{k-2} \|s_j - s_{j+1}\|
    + \|s_{k-1} - x_{\text{goal}}\|.
\]

For any two waypoint vectors $\mathbf{s},\mathbf{s}'$,
each term in $c(\sigma_{\mathbf{s}})$ is $1$-Lipschitz in each of its
arguments. By the triangle inequality (e.g.,~\cite{horn2012matrix}),
\[
  |c(\sigma_{\mathbf{s}}) - c(\sigma_{\mathbf{s}'})|
  \le 2 \sum_{j=1}^{k-1} \|s_j - s'_j\|
  \le 2(k-1)\,\|\mathbf{s} - \mathbf{s}'\|_\infty,
\]
where $\|\mathbf{s}\|_\infty := \max_{j=1,\dots,k-1}\|s_j\|$.
Thus $c(\sigma_{\mathbf{s}})$ is Lipschitz with constant
$L = 2(k-1)$.

By uniform sampling on each constrained zonotope face $\mathcal{F}_{i_j,i_{j+1}}$ and the strong law of large numbers for empirical measures on compact sets~\cite{dudley2018real, dudley1966weak},
\[
  \delta_N 
  := \max_{j=1,\ldots,k-1} \min_{l=1,\ldots,N_s} \|s_j^{(l)} - s_j^*\|
  \;\xrightarrow[N_s \to \infty]{\text{a.s.}}\; 0,
\]
where $s_j^*$ denotes the optimal waypoint on the $j$-th face.
By Lipschitz continuity,
\[
  |c_{\text{best}} - c^*|
  \;\le\; L\,\delta_N
  \;\xrightarrow[N_s \to \infty]{\text{a.s.}}\; 0,
\]
so $c_{\text{best}} \to c^*$ almost surely and 
$\lim_{N_s \to \infty} \mathbb{E}[c_{\text{best}}] = c^*$.




\end{proof}

\section{Numerical examples}\label{sec-numericalEx}

The proposed hybrid zonotope-based motion planning algorithm is evaluated on challenging scenarios that demonstrate its effectiveness in complex environments with narrow passages and enclosures. All experiments are implemented in MATLAB and executed on a Lenovo laptop equipped with an Intel 64-bit processor (1.5\,GHz) and 32\,GB of RAM on a single CPU core, without GPU acceleration.

\subsection{Convergence analysis}

The algorithm is evaluated on a planar maze environment $\mathcal{X} = [-100, 800] \times [-100, 600]$ containing 13 obstacles: rectangular barriers, polygonal approximations of circles, U-shaped regions, star polygons, and narrow passages (60\,cm width). The start state is $x_{\text{start}} = (400, 250)^T$ and the goal state is $x_{\text{goal}} = (385, 470)^T$, requiring navigation through the maze structure and around multiple obstacles. This environment exemplifies the challenges of motion planning in cluttered spaces with both narrow passages and complex obstacle geometries.

\begin{figure}[t]
    \centering
    \includegraphics[width=0.8\linewidth]{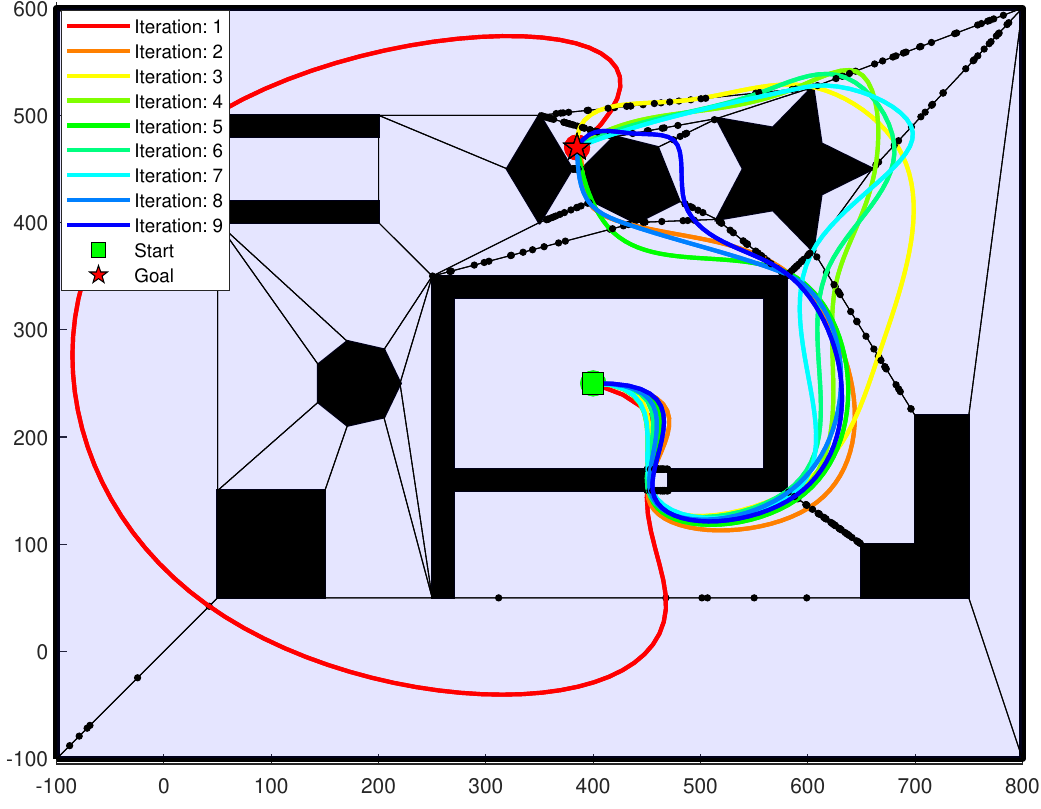}
    \caption{Motion planning scenario depicting the environment with obstacles (gray), the start state (green), the goal (red), and the solution path obtained by the proposed algorithm.}
    \label{fig:iteration_plot}
\end{figure}
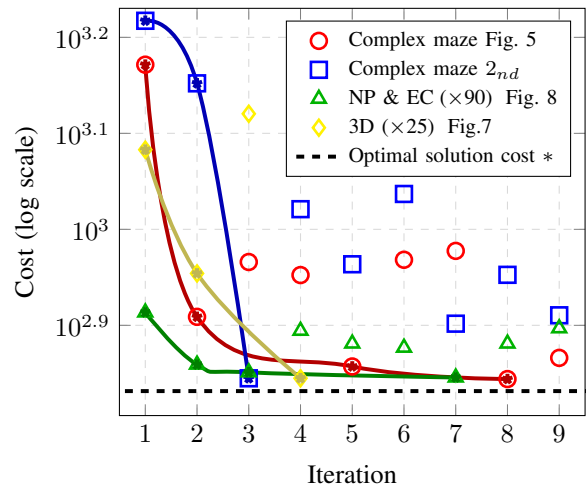
\begin{figure}[t]
    \centering
    \begin{tikzpicture}
    \begin{semilogyaxis}[
        width=0.9\linewidth,
        height=5.5cm,
        xlabel={Iteration},
        ylabel={Cost (log scale)},
        xmin=0.5, xmax=9.5,
        ymin=640, ymax=1700,
        xtick={1,2,3,4,5,6,7,8,9},
        grid=major,
        grid style={dashed, gray!30},
        legend pos=north east,
        legend style={font=\footnotesize, cells={anchor=west}},
    ]
    
    \addplot[
        only marks,
        mark=o,
        mark size=3pt,
        color=red,
        line width=1pt
    ] coordinates {
        (1, 1485)
        (2, 811)
        (3, 925)
        (4, 896.5)
        (5, 720)
        (6, 930)
        (7, 950)
        (8, 698.603)
        (9, 735)
    };
    
    \addplot[
        color=red!70!black,
        line width=1.5pt,
        smooth,
        tension=0.7,
        mark=star,
        mark size=2pt,
    ] coordinates {
        (1, 1485)
        (2, 811)
        (5, 720)
        (8, 698.603)
    };
    
    \addplot[
        only marks,
        mark=square,
        mark size=3pt,
        color=blue,
        line width=1pt
    ] coordinates {
        (1, 1650)
        (2, 1420)
        (3, 699.8)
        (4, 1050)
        (5, 920)
        (6, 1089)
        (7, 798)
        (8, 897)
        (9, 814)
    };
    
    \addplot[
        color=blue!70!black,
        line width=1.5pt,
        smooth,
        tension=0.7,
        mark=star,
        mark size=2pt,
    ] coordinates {
        (1, 1650)
        (2, 1420)
        (3, 699.8)
    };
    
    \addplot[
        only marks,
        mark=triangle,
        mark size=3pt,
        color=green!70!black,
        line width=1pt
    ] coordinates {
       (1, 819.58)
(2, 723.27)
(3, 709.96)
(4, 784.46)
(5, 760.79)
(6, 753.28)
(7, 701.21)
(8, 760.73)
(9, 789.18)
    };
    
    \addplot[
        color=green!50!black,
        line width=1.5pt,
        smooth,
        tension=0.7,
        mark=star,
        mark size=2pt,
    ] coordinates {
        (1, 819.58)
(2, 723.27)
(3, 709.96)
(7, 701.21)
    };
    
    \addplot[
        only marks,
        mark=diamond,
        mark size=3pt,
        color=yellow,
        line width=1pt
    ] coordinates {
        (1, 1211)
        (2, 900)
        (3, 1320)
        (4, 700)
    };
    
    \addplot[
        color=yellow!70!black,
        line width=1.5pt,
        smooth,
        tension=0.7,
        mark=star,
        mark size=2pt,
    ] coordinates {
        (1, 1211)
        (2, 900)
        (4, 700)
    };

    \addplot[
        color=black!70!black,
        line width=1.5pt,
        dashed,
    ] coordinates {
        (0.3, 678.7)
        (9.7, 678.7)
    };
    
    \legend{
        Complex maze Fig.~\ref{fig:iteration_plot},
        ,
        Complex maze $2_{nd}$,
        ,
        NP \& EC\;($\times$90)\; Fig.~\ref{fig:narrow_gap},
        ,
        3D ($\times$25)\; Fig.\ref{fig:3d_narrow},
        ,
        Optimal solution cost $*$
    }
    
    \end{semilogyaxis}
    \end{tikzpicture}
    \caption{Convergence analysis across different environments illustrating the stochastic nature of the sampling-based approach. Markers indicate costs at each iteration; stars connected by smooth curves denote iterations at which a better solution was found and adopted. The non-monotonic behavior between starred points reflects the exploration--exploitation trade-off inherent in sampling-based methods. For visualization purposes, costs for the NP \& EC and 3D scenarios are scaled by factors of 90 and 25, respectively.}
    \label{fig:cost_convergence_comparison}
\end{figure}

Figure~\ref{fig:iteration_plot} illustrates the convergence behavior of the algorithm in a complex 2D environment. Over 8 iterations, significant solution improvements occur at iterations 1, 2, 5, and 8 (marked with green stars), yielding a final cost of 698.603. The cost exhibits a rapid initial decrease followed by stabilization, consistent with the theoretical convergence guarantees. The hybrid zonotope decomposition enables systematic exploration of multiple homotopy classes, thereby efficiently identifying the globally optimal path.

\subsection{3D narrow passage navigation}

The algorithm is further evaluated in a three-dimensional environment $\mathcal{X} = [-15, 15] \times [0, 15] \times [0, 5] \subset \mathbb{R}^3$. The environment features a central wall at $(0, 0, 2.5)^T$ spanning 13\,m vertically with a critical narrow passage of 1\,m height at $z = 4.5$\,m. All walls have thickness 0.3\,m, creating a challenging topology in which the optimal path must navigate through aligned openings. The start state is $x_{\text{start}} = (-13,10,3)^T$ and the goal state is $x_{\text{goal}} = (13,3,2)^T$, requiring navigation through the narrow gap and thereby demonstrating the capability of the algorithm in constrained 3D spaces.

\begin{figure}[t]
    \centering
    \includegraphics[width=0.7\linewidth]{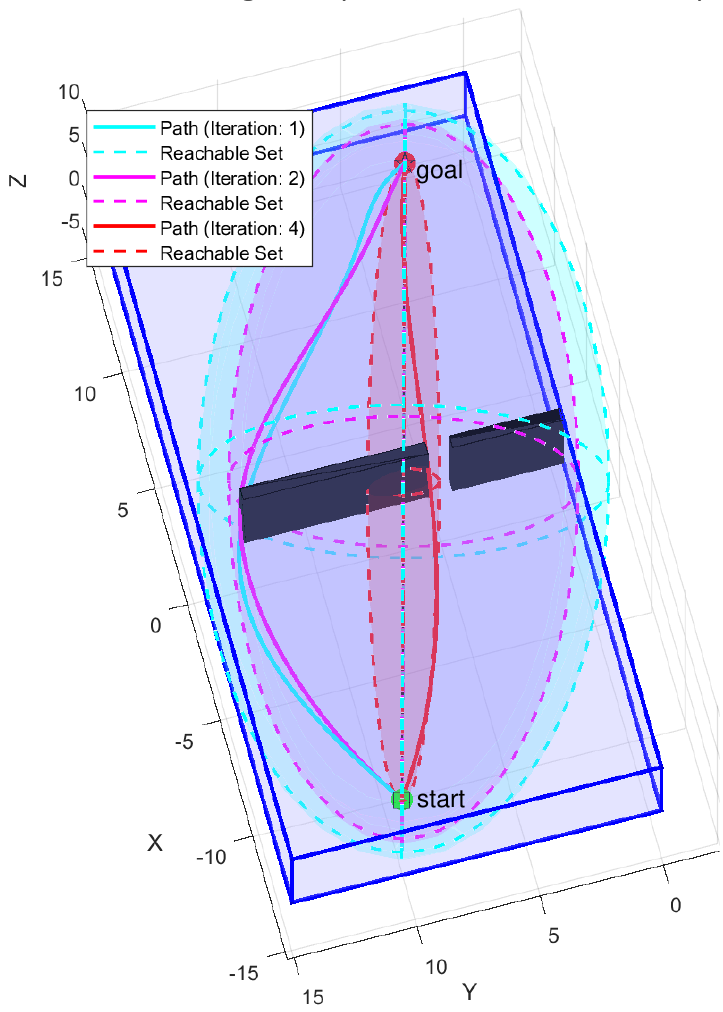}
    \caption{3D narrow passage scenario. Different colored paths and ellipsoidal reachable sets correspond to successive iterations; the green cube is the start, the red sphere is the goal.}
    \label{fig:3d_narrow}
\end{figure}

Figure~\ref{fig:3d_narrow} illustrates the iterative refinement. The cyan path (iteration~1) follows a conservative route with a large reachable ellipsoid; the magenta path (iteration~2) improves with a reduced set; and the red path (iteration~4) identifies the optimal trajectory through the narrow passage. The progressively shrinking ellipsoids demonstrate the effectiveness of the informed sampling strategy.
\subsection{Narrow passages and enclosure example}

The algorithm is further evaluated on a planar enclosure environment $\mathcal{X} = [0, 10] \times [0, 10] \subset \mathbb{R}^2$ featuring a rectangular enclosure with interior dimensions of approximately $6 \times 4$ units centered at $(5, 4)^T$. The enclosure walls have thickness 0.3 units and contain a critical narrow opening of width 0.8 units on the left wall at height $y = 4$. The start state $x_{\text{start}} = (3, 3)^T$ is positioned outside the enclosure, while the goal state $x_{\text{goal}} = (7, 5)^T$ lies within the enclosed region, necessitating passage through the narrow opening. This configuration exemplifies the narrow passage problem, in which the measure of the feasible region connecting start to goal is vanishingly small relative to the total free-space volume.

\begin{figure}[t]
    \centering
    \includegraphics[width=0.55\linewidth]{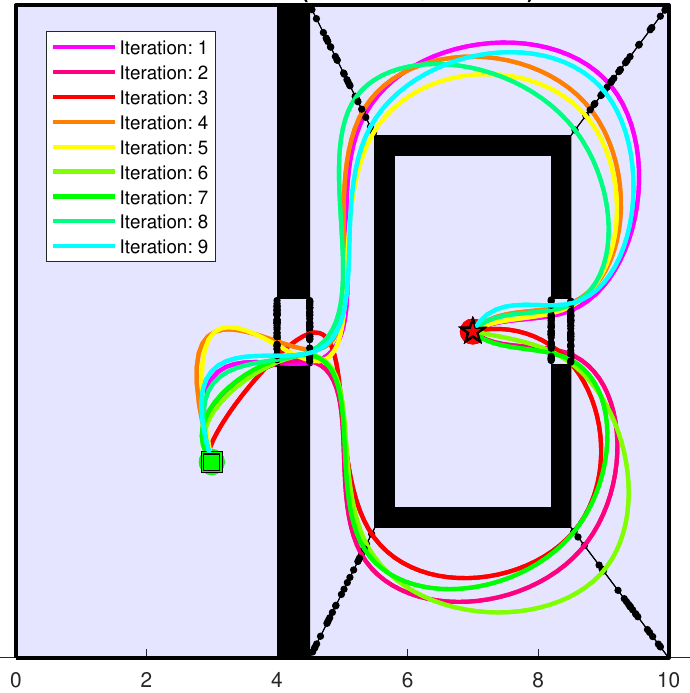}
    \caption{Narrow passage and enclosure scenario. Numbered regions (0, 1, 2) represent leaf nodes; region 012 indicates the critical shared face at the narrow passage.}
    \label{fig:narrow_gap}
\end{figure}

Figure~\ref{fig:narrow_gap} demonstrates the capability of the algorithm in extremely constrained environments. The hybrid zonotope decomposition explicitly captures the narrow passage as a shared face between leaf nodes, ensuring high sampling density precisely where it is needed. By sampling on $(n{-}1)$-dimensional shared faces rather than $n$-dimensional volumes, the proposed approach guarantees discovery of feasible paths through narrow openings.


This experiment is conducted using the Open Motion Planning Library (OMPL)~\cite{sucan2012open} and the Planner Developer Tools (PDT)~\cite{gammell2022planner}. A total of 100 runs of Informed RRT*, BIT*~\cite{gammell2020batch}, AIT*, EIT*, and HZ-MP are executed on the \textit{defaultRandomRectangles2D} planning context.


\begin{table}[t]
  \centering
   \caption{Summary statistics over 100 runs on the \textit{defaultRandomRectangles2D} benchmark.}
  \label{tab:random-rectangles-summary}
\resizebox{\columnwidth}{!}{%
\scriptsize
\setlength{\tabcolsep}{0.35em}%
\begin{tabular}{lcccccccccc}
  \toprule
  Planner &
  $t_\mathrm{init}^\mathrm{min}$ &
  $t_\mathrm{init}^\mathrm{med}$ &
  $t_\mathrm{init}^\mathrm{max}$ &
  $c_\mathrm{init}^\mathrm{min}$ &
  $c_\mathrm{init}^\mathrm{med}$ &
  $c_\mathrm{init}^\mathrm{max}$ &
  $c_\mathrm{final}^\mathrm{min}$ &
  $c_\mathrm{final}^\mathrm{med}$ &
  $c_\mathrm{final}^\mathrm{max}$ &
  Succ. \\
  \midrule
  In-RRT* &
    0.041 & $\infty$ & $\infty$ &
    0.971 & $\infty$ & $\infty$ &
    0.971 & $\infty$ & $\infty$ &
    0.12 \\
  BIT* &
    0.045 & 0.107 & $\infty$ &
    0.906 & 1.175 & $\infty$ &
    0.855 & 1.043 & $\infty$ &
    0.62 \\
  AIT* &
    0.053 & 0.129 & $\infty$ &
    0.898 & 1.069 & $\infty$ &
    0.864 & 1.011 & $\infty$ &
    0.71 \\
  EIT* &
    0.021 & 0.028 & 0.123 &
    \bfseries 0.873 & \bfseries 1.056 & \bfseries 1.740 &
    \bfseries 0.819 & \bfseries 0.858 & \bfseries 0.983 &
    \bfseries 1.00 \\
  HZ-MP &
    \bfseries 0.002 & \bfseries 0.003 & \bfseries 0.003 &
    1.445 & 2.255 & 3.552 &
    1.007 & 1.007 & 1.007 &
    \bfseries 1.00 \\
  \bottomrule
\end{tabular}%
}%

\vspace{0.5ex}
\noindent\footnotesize\emph{Note.} In-RRT* denotes Informed RRT*. Succ.\ denotes the success rate. All values rounded to three decimal places.

\end{table}

\begin{figure}[t]
  \centering
  \resizebox{\columnwidth}{!}{%
    \input{tikz/SuccessAxis_AllPlannersMedianCostAxis__.tikz}%
  }%
  \caption{\footnotesize
    \textbf{Top:} Percentage of runs that found a solution at any given time,
    with Clopper--Pearson (nonparametric) 99\% confidence intervals.
    \textbf{Bottom:} Median cost evolution and median initial solution cost,
    again with nonparametric 99\% confidence intervals.
  }
  \label{fig:overview-success-cost}
\end{figure}

Table~\ref{tab:random-rectangles-summary} and Figure~\ref{fig:overview-success-cost} summarize the results. HZ-MP achieves the fastest initial solution times by more than one order of magnitude compared with the sampling-based baselines while maintaining a $100\%$ success rate; Informed RRT* frequently fails within the time budget, as indicated by its infinite medians. EIT* also reaches $100\%$ success and yields the lowest median final cost, reflecting its effectiveness at anytime refinement. HZ-MP trades this long-horizon refinement quality for highly structured exploration: by sampling exclusively on $(n{-}1)$-dimensional shared faces and pruning via the ellipsotope, it reaches a near-optimal solution orders of magnitude faster, making it particularly suitable for settings in which low-latency initial solutions are critical.

\section{Conclusion}\label{sec-conclusion}

This paper presented HZ-MP, a motion planning algorithm that samples on $(n{-}1)$-dimensional shared faces of a hybrid zonotope decomposition, directly addressing the narrow passage problem while preserving probabilistic completeness and asymptotic optimality. Ellipsotope-informed pruning accelerates convergence, as confirmed by experiments in challenging 2D and 3D scenarios. Future work will extend the framework to dynamic environments and kinodynamic constraints.

\balance
\bibliographystyle{ACM-Reference-Format}
\bibliography{BibTex_2025}


\begin{thebibliography}{26}


\ifx \showCODEN    \undefined \def \showCODEN     #1{\unskip}     \fi
\ifx \showISBNx    \undefined \def \showISBNx     #1{\unskip}     \fi
\ifx \showISBNxiii \undefined \def \showISBNxiii  #1{\unskip}     \fi
\ifx \showISSN     \undefined \def \showISSN      #1{\unskip}     \fi
\ifx \showLCCN     \undefined \def \showLCCN      #1{\unskip}     \fi
\ifx \shownote     \undefined \def \shownote      #1{#1}          \fi
\ifx \showarticletitle \undefined \def \showarticletitle #1{#1}   \fi
\ifx \showURL      \undefined \def \showURL       {\relax}        \fi
\providecommand\bibfield[2]{#2}
\providecommand\bibinfo[2]{#2}
\providecommand\natexlab[1]{#1}
\providecommand\showeprint[2][]{arXiv:#2}

\bibitem[Bird et~al\mbox{.}(2023)]%
        {Bird2023}
\bibfield{author}{\bibinfo{person}{Trevor~J. Bird},
  \bibinfo{person}{Herschel~C. Pangborn}, \bibinfo{person}{Neera Jain}, {and}
  \bibinfo{person}{Justin~P. Koeln}.} \bibinfo{year}{2023}\natexlab{}.
\newblock \showarticletitle{Hybrid zonotopes: A new set representation for
  reachability analysis of mixed logical dynamical systems}.
\newblock \bibinfo{journal}{\emph{Automatica}}  \bibinfo{volume}{154}
  (\bibinfo{year}{2023}), \bibinfo{pages}{111107}.
\newblock
\href{https://doi.org/10.1016/j.automatica.2023.111107}{doi:\nolinkurl{10.1016/j.automatica.2023.111107}}


\bibitem[Dudley(1966)]%
        {dudley1966weak}
\bibfield{author}{\bibinfo{person}{Richard~Mansfield Dudley}.}
  \bibinfo{year}{1966}\natexlab{}.
\newblock \showarticletitle{Weak convergence of probabilities on nonseparable
  metric spaces and empirical measures on Euclidean spaces}.
\newblock \bibinfo{journal}{\emph{Illinois Journal of Mathematics}}
  \bibinfo{volume}{10}, \bibinfo{number}{1} (\bibinfo{year}{1966}),
  \bibinfo{pages}{109--126}.
\newblock


\bibitem[Dudley(2002)]%
        {dudley2018real}
\bibfield{author}{\bibinfo{person}{Richard~M Dudley}.}
  \bibinfo{year}{2002}\natexlab{}.
\newblock \bibinfo{booktitle}{\emph{Real Analysis and Probability}
  (\bibinfo{edition}{2nd} ed.)}.
\newblock \bibinfo{publisher}{Cambridge University Press}.
\newblock


\bibitem[Eppstein(1998)]%
        {eppstein1998finding}
\bibfield{author}{\bibinfo{person}{David Eppstein}.}
  \bibinfo{year}{1998}\natexlab{}.
\newblock \showarticletitle{Finding the k shortest paths}.
\newblock \bibinfo{journal}{\emph{SIAM Journal on computing}}
  \bibinfo{volume}{28}, \bibinfo{number}{2} (\bibinfo{year}{1998}),
  \bibinfo{pages}{652--673}.
\newblock


\bibitem[Gammell et~al\mbox{.}(2020)]%
        {gammell2020batch}
\bibfield{author}{\bibinfo{person}{Jonathan~D Gammell},
  \bibinfo{person}{Timothy~D Barfoot}, {and} \bibinfo{person}{Siddhartha~S
  Srinivasa}.} \bibinfo{year}{2020}\natexlab{}.
\newblock \showarticletitle{Batch Informed Trees ({BIT*}): Informed
  Asymptotically Optimal Anytime Search}.
\newblock \bibinfo{journal}{\emph{The International Journal of Robotics
  Research}} \bibinfo{volume}{39}, \bibinfo{number}{5} (\bibinfo{year}{2020}),
  \bibinfo{pages}{543--567}.
\newblock
\href{https://doi.org/10.1177/0278364919890396}{doi:\nolinkurl{10.1177/0278364919890396}}


\bibitem[Gammell et~al\mbox{.}(2014)]%
        {gammell_informed_2014}
\bibfield{author}{\bibinfo{person}{Jonathan~D. Gammell},
  \bibinfo{person}{Siddhartha~S. Srinivasa}, {and} \bibinfo{person}{Timothy~D.
  Barfoot}.} \bibinfo{year}{2014}\natexlab{}.
\newblock \showarticletitle{Informed RRT*: Optimal Sampling-based Path Planning
  Focused via Direct Sampling of an Admissible Ellipsoidal Heuristic}. In
  \bibinfo{booktitle}{\emph{IEEE/RSJ International Conference on Intelligent
  Robots and Systems}}. IEEE, \bibinfo{pages}{2997--3004}.
\newblock


\bibitem[Gammell et~al\mbox{.}(2022)]%
        {gammell2022planner}
\bibfield{author}{\bibinfo{person}{Jonathan~D Gammell},
  \bibinfo{person}{Marlin~P Strub}, {and} \bibinfo{person}{Valentin~N
  Hartmann}.} \bibinfo{year}{2022}\natexlab{}.
\newblock \showarticletitle{Planner developer tools (pdt): Reproducible
  experiments and statistical analysis for developing and testing motion
  planners}. In \bibinfo{booktitle}{\emph{Proceedings of the Workshop on
  Evaluating Motion Planning Performance (EMPP), IEEE/RSJ International
  Conference on Intelligent Robots and Systems (IROS)}}.
\newblock


\bibitem[Hachenberger et~al\mbox{.}(2007)]%
        {hachenberger2007boolean}
\bibfield{author}{\bibinfo{person}{Peter Hachenberger}, \bibinfo{person}{Lutz
  Kettner}, {and} \bibinfo{person}{Kurt Mehlhorn}.}
  \bibinfo{year}{2007}\natexlab{}.
\newblock \showarticletitle{Boolean operations on {3D} selective {Nef}
  complexes: Data structure, algorithms, optimized implementation and
  experiments}.
\newblock \bibinfo{journal}{\emph{Computational Geometry}}
  \bibinfo{volume}{38}, \bibinfo{number}{1-2} (\bibinfo{year}{2007}),
  \bibinfo{pages}{64--99}.
\newblock


\bibitem[Horn and Johnson(2012)]%
        {horn2012matrix}
\bibfield{author}{\bibinfo{person}{Roger~A Horn} {and}
  \bibinfo{person}{Charles~R Johnson}.} \bibinfo{year}{2012}\natexlab{}.
\newblock \bibinfo{booktitle}{\emph{Matrix Analysis} (\bibinfo{edition}{2nd}
  ed.)}.
\newblock \bibinfo{publisher}{Cambridge University Press}.
\newblock


\bibitem[Ioan et~al\mbox{.}(2021)]%
        {Ioan2021}
\bibfield{author}{\bibinfo{person}{Daniel Ioan}, \bibinfo{person}{Ionela
  Prodan}, \bibinfo{person}{Sorin Olaru}, \bibinfo{person}{Florin Stoican},
  {and} \bibinfo{person}{Silviu~I. Niculescu}.}
  \bibinfo{year}{2021}\natexlab{}.
\newblock \showarticletitle{Mixed-integer programming in motion planning}.
\newblock \bibinfo{journal}{\emph{Annual Reviews in Control}}
  \bibinfo{volume}{51} (\bibinfo{year}{2021}), \bibinfo{pages}{65--87}.
\newblock
\href{https://doi.org/10.1016/j.arcontrol.2020.10.008}{doi:\nolinkurl{10.1016/j.arcontrol.2020.10.008}}


\bibitem[Karaman and Frazzoli(2011)]%
        {Karaman2011}
\bibfield{author}{\bibinfo{person}{Sertac Karaman} {and}
  \bibinfo{person}{Emilio Frazzoli}.} \bibinfo{year}{2011}\natexlab{}.
\newblock \showarticletitle{Sampling-based Algorithms for Optimal Motion
  Planning}.
\newblock \bibinfo{journal}{\emph{The International Journal of Robotics
  Research}} \bibinfo{volume}{30}, \bibinfo{number}{7} (\bibinfo{year}{2011}),
  \bibinfo{pages}{846--894}.
\newblock
\href{https://doi.org/10.1177/0278364911406761}{doi:\nolinkurl{10.1177/0278364911406761}}


\bibitem[Kavraki et~al\mbox{.}(1996)]%
        {Kavraki1996}
\bibfield{author}{\bibinfo{person}{Lydia~E. Kavraki}, \bibinfo{person}{Petr
  {\v{S}}vestka}, \bibinfo{person}{Jean-Claude Latombe}, {and}
  \bibinfo{person}{Mark~H. Overmars}.} \bibinfo{year}{1996}\natexlab{}.
\newblock \showarticletitle{Probabilistic roadmaps for path planning in
  high-dimensional configuration spaces}.
\newblock \bibinfo{journal}{\emph{IEEE Transactions on Robotics and
  Automation}} \bibinfo{volume}{12}, \bibinfo{number}{4}
  (\bibinfo{year}{1996}), \bibinfo{pages}{566--580}.
\newblock


\bibitem[Koeln et~al\mbox{.}(2024)]%
        {koeln2024zonolab}
\bibfield{author}{\bibinfo{person}{Justin Koeln}, \bibinfo{person}{Trevor~J
  Bird}, \bibinfo{person}{Jacob Siefert}, \bibinfo{person}{Justin Ruths},
  \bibinfo{person}{Herschel~C Pangborn}, {and} \bibinfo{person}{Neera Jain}.}
  \bibinfo{year}{2024}\natexlab{}.
\newblock \showarticletitle{zonoLAB: A MATLAB toolbox for set-based control
  systems analysis using hybrid zonotopes}. In \bibinfo{booktitle}{\emph{2024
  American Control Conference (ACC)}}. IEEE, \bibinfo{pages}{2513--2520}.
\newblock


\bibitem[Kousik et~al\mbox{.}(2023)]%
        {kousik2023}
\bibfield{author}{\bibinfo{person}{Shreyas Kousik}, \bibinfo{person}{Adam Dai},
  {and} \bibinfo{person}{Grace~Xingxin Gao}.} \bibinfo{year}{2023}\natexlab{}.
\newblock \showarticletitle{Ellipsotopes: Uniting Ellipsoids and Zonotopes for
  Reachability Analysis and Fault Detection}.
\newblock \bibinfo{journal}{\emph{IEEE Trans. Automat. Control}}
  \bibinfo{volume}{68}, \bibinfo{number}{6} (\bibinfo{year}{2023}),
  \bibinfo{pages}{3440--3452}.
\newblock
\href{https://doi.org/10.1109/TAC.2022.3191750}{doi:\nolinkurl{10.1109/TAC.2022.3191750}}


\bibitem[Kulmburg et~al\mbox{.}(2024)]%
        {holmes2023searchbased}
\bibfield{author}{\bibinfo{person}{Adrian Kulmburg}, \bibinfo{person}{Ivan
  Brkan}, {and} \bibinfo{person}{Matthias Althoff}.}
  \bibinfo{year}{2024}\natexlab{}.
\newblock \showarticletitle{Search-based and Stochastic Solutions to the
  Zonotope and Ellipsotope Containment Problems}. In
  \bibinfo{booktitle}{\emph{2024 European Control Conference (ECC)}}.
  \bibinfo{pages}{1057--1064}.
\newblock
\href{https://doi.org/10.23919/ECC64448.2024.10590884}{doi:\nolinkurl{10.23919/ECC64448.2024.10590884}}


\bibitem[Orthey and Toussaint(2021)]%
        {orthey2021section}
\bibfield{author}{\bibinfo{person}{Andreas Orthey} {and} \bibinfo{person}{Marc
  Toussaint}.} \bibinfo{year}{2021}\natexlab{}.
\newblock \showarticletitle{Section patterns: Efficiently solving narrow
  passage problems in multilevel motion planning}.
\newblock \bibinfo{journal}{\emph{IEEE Transactions on Robotics}}
  \bibinfo{volume}{37}, \bibinfo{number}{6} (\bibinfo{year}{2021}),
  \bibinfo{pages}{1891--1905}.
\newblock


\bibitem[Robbins et~al\mbox{.}(2024)]%
        {Robbins2024}
\bibfield{author}{\bibinfo{person}{Joshua~A. Robbins}, \bibinfo{person}{Sean~B.
  Brennan}, {and} \bibinfo{person}{Herschel~C. Pangborn}.}
  \bibinfo{year}{2024}\natexlab{}.
\newblock \showarticletitle{Efficient Solution of Mixed-Integer {MPC} Problems
  for Obstacle Avoidance Using Hybrid Zonotopes}. In
  \bibinfo{booktitle}{\emph{Proc. 63rd IEEE Conf. on Decision and Control
  (CDC)}}. \bibinfo{pages}{8199--8206}.
\newblock
\href{https://doi.org/10.1109/CDC56724.2024.10886569}{doi:\nolinkurl{10.1109/CDC56724.2024.10886569}}


\bibitem[Robbins et~al\mbox{.}(2026)]%
        {robbins2024mixed}
\bibfield{author}{\bibinfo{person}{Joshua~A Robbins}, \bibinfo{person}{Jacob~A
  Siefert}, \bibinfo{person}{Sean Brennan}, {and} \bibinfo{person}{Herschel~C
  Pangborn}.} \bibinfo{year}{2026}\natexlab{}.
\newblock \showarticletitle{Mixed-Integer MPC-Based Motion Planning Using
  Hybrid Zonotopes with Tight Relaxations}.
\newblock \bibinfo{journal}{\emph{IEEE Transactions on Control Systems
  Technology}} (\bibinfo{year}{2026}).
\newblock
\href{https://doi.org/10.1109/TCST.2026.3661166}{doi:\nolinkurl{10.1109/TCST.2026.3661166}}
\newblock
\shownote{Early access}.


\bibitem[Scott et~al\mbox{.}(2016)]%
        {scott_constrained_2016}
\bibfield{author}{\bibinfo{person}{Joseph~K. Scott}, \bibinfo{person}{Davide~M.
  Raimondo}, \bibinfo{person}{Giuseppe~R. Marseglia}, {and}
  \bibinfo{person}{Richard~D. Braatz}.} \bibinfo{year}{2016}\natexlab{}.
\newblock \showarticletitle{Constrained zonotopes: A new tool for set-based
  estimation and fault detection}.
\newblock \bibinfo{journal}{\emph{Automatica}}  \bibinfo{volume}{69}
  (\bibinfo{year}{2016}), \bibinfo{pages}{126--136}.
\newblock


\bibitem[Shoja et~al\mbox{.}(2022)]%
        {shoja2022overall}
\bibfield{author}{\bibinfo{person}{Shamisa Shoja}, \bibinfo{person}{Daniel
  Arnstr{\"o}m}, {and} \bibinfo{person}{Daniel Axehill}.}
  \bibinfo{year}{2022}\natexlab{}.
\newblock \showarticletitle{Overall complexity certification of a standard
  branch and bound method for mixed-integer quadratic programming}. In
  \bibinfo{booktitle}{\emph{2022 American Control Conference (ACC)}}. IEEE,
  \bibinfo{pages}{4957--4964}.
\newblock


\bibitem[Silvestre(2022)]%
        {ccg2022}
\bibfield{author}{\bibinfo{person}{Daniel Silvestre}.}
  \bibinfo{year}{2022}\natexlab{}.
\newblock \showarticletitle{Constrained Convex Generators: A Tool Suitable for
  Set-Based Estimation With Range and Bearing Measurements}.
\newblock \bibinfo{journal}{\emph{IEEE Control Systems Letters}}
  \bibinfo{volume}{6} (\bibinfo{year}{2022}), \bibinfo{pages}{1610--1615}.
\newblock
\href{https://doi.org/10.1109/LCSYS.2021.3129729}{doi:\nolinkurl{10.1109/LCSYS.2021.3129729}}


\bibitem[Strub and Gammell(2022)]%
        {strub_ijrr22}
\bibfield{author}{\bibinfo{person}{Marlin~P. Strub} {and}
  \bibinfo{person}{Jonathan~D. Gammell}.} \bibinfo{year}{2022}\natexlab{}.
\newblock \showarticletitle{Adaptively Informed Trees ({AIT*}) and Effort
  Informed Trees ({EIT*}): Asymmetric bidirectional sampling-based path
  planning}.
\newblock \bibinfo{journal}{\emph{Int. Journal of Robotics Research}}
  \bibinfo{volume}{41}, \bibinfo{number}{4} (\bibinfo{year}{2022}),
  \bibinfo{pages}{390--417}.
\newblock
\href{https://doi.org/10.1177/02783649211069572}{doi:\nolinkurl{10.1177/02783649211069572}}


\bibitem[Sucan et~al\mbox{.}(2012)]%
        {sucan2012open}
\bibfield{author}{\bibinfo{person}{Ioan~A Sucan}, \bibinfo{person}{Mark Moll},
  {and} \bibinfo{person}{Lydia~E Kavraki}.} \bibinfo{year}{2012}\natexlab{}.
\newblock \showarticletitle{The open motion planning library}.
\newblock \bibinfo{journal}{\emph{IEEE Robotics \& Automation Magazine}}
  \bibinfo{volume}{19}, \bibinfo{number}{4} (\bibinfo{year}{2012}),
  \bibinfo{pages}{72--82}.
\newblock


\bibitem[Szkandera et~al\mbox{.}(2020)]%
        {szkandera2020narrow}
\bibfield{author}{\bibinfo{person}{Jakub Szkandera}, \bibinfo{person}{Ivana
  Kolingerov{\'a}}, {and} \bibinfo{person}{Martin Ma{\v{n}}{\'a}k}.}
  \bibinfo{year}{2020}\natexlab{}.
\newblock \showarticletitle{Narrow Passage Problem Solution for Motion
  Planning}. In \bibinfo{booktitle}{\emph{International Conference on
  Computational Science (ICCS)}} \emph{(\bibinfo{series}{Lecture Notes in
  Computer Science}, Vol.~\bibinfo{volume}{12140})}.
  \bibinfo{publisher}{Springer}, \bibinfo{pages}{459--470}.
\newblock
\href{https://doi.org/10.1007/978-3-030-50371-0_34}{doi:\nolinkurl{10.1007/978-3-030-50371-0_34}}


\bibitem[{The CGAL Project}(2024)]%
        {cgal:eb-24b}
\bibfield{author}{\bibinfo{person}{{The CGAL Project}}.}
  \bibinfo{year}{2024}\natexlab{}.
\newblock \bibinfo{booktitle}{\emph{{CGAL} User and Reference Manual}
  (\bibinfo{edition}{{6.0.1}} ed.)}.
\newblock \bibinfo{publisher}{{CGAL Editorial Board}}.
\newblock
\urldef\tempurl%
\url{https://doc.cgal.org/6.0.1/Manual/packages.html}
\showURL{%
\tempurl}


\bibitem[Yen(1971)]%
        {yen1971finding}
\bibfield{author}{\bibinfo{person}{Jin~Y Yen}.}
  \bibinfo{year}{1971}\natexlab{}.
\newblock \showarticletitle{Finding the k shortest loopless paths in a
  network}.
\newblock \bibinfo{journal}{\emph{Management Science}} \bibinfo{volume}{17},
  \bibinfo{number}{11} (\bibinfo{year}{1971}), \bibinfo{pages}{712--716}.
\newblock


\end{thebibliography}
\end{document}